\pdfoutput=1

\documentclass[10pt,letterpaper]{article}

\usepackage[utf8]{inputenc} % allow utf-8 input
\usepackage[T1]{fontenc}    % use 8-bit T1 fonts
\usepackage{lmodern}
\usepackage{hyperref}       % hyperlinks
\usepackage{url}            % simple URL typesetting
\usepackage{booktabs}       % professional-quality tables
\usepackage{amsfonts}       % blackboard math symbols
\usepackage{nicefrac}       % compact symbols for 1/2, etc.
\usepackage{microtype}      % microtypography
\usepackage{xspace}
\usepackage{authblk}

\usepackage{graphicx}
\graphicspath{{pictures/}}
\usepackage{amsmath}
\usepackage{amssymb}
\usepackage{amsthm}
\usepackage{bm}
\usepackage[margin=1in]{geometry}

\usepackage[font=small]{caption}
\usepackage{subcaption}
\usepackage{tikz}
\usepackage{pgfplots}
\usepackage{fp}
\usetikzlibrary{bayesnet}

\theoremstyle{plain}
\newtheorem{thm}{Theorem}[section]

\newtheorem*{cor}{Corollary}

\theoremstyle{definition}
\newtheorem{defn}{Definition}[section]

\theoremstyle{remark}

\DeclareMathOperator*{\argmin}{arg\,min}

\newcommand{\norm}[1] {\left\| #1 \right\|}
\newcommand{\abs}[1] {| #1 |}
\newcommand{\normtwo}[1] {\left\| #1 \right\|_2}
\newcommand{\normf}[1] {\left\| #1 \right\|_{\mathcal{F}}}

\newcommand{\sumi}{\sum_{i=1}^m}
\newcommand{\sump}{\sum_{p=1}^L}

\newcommand{\sumik}[1]{\sum_{#1=1 | k_{#1}=k}^m}

\newcommand{\mul}[1] {{\mu}_{\mkern-0.66\thinmuskip\mathcal{L}} (#1) }

\newcommand{\wrt}{w.r.t.\@\xspace}

\newcommand\todo[1]{\textcolor{red}{#1}}
\renewcommand\todo[1]{}

\newcommand{\twocol}[2]{
    \begin{tabular}{cc}
        #1 & #2 
    \end{tabular}
}

\newcommand{\landSVM}{L$^3$-SVMs\xspace}

\author[1]{Zantedeschi Valentina \\
\href{mailto:valentina.zantedeschi@univ-st-etienne.fr}{valentina.zantedeschi@univ-st-etienne.fr}}
\author[1]{\\Emonet R\'emi\\
\href{mailto:remi.emonet@univ-st-etienne.fr}{remi.emonet@univ-st-etienne.fr}}
\author[1]{\\Sebban Marc\\
\href{mailto:marc.sebban@univ-st-etienne.fr}{marc.sebban@univ-st-etienne.fr}}

\affil[1]{
Univ Lyon, UJM-Saint-Etienne, CNRS, Institut d Optique Graduate School, Laboratoire Hubert Curien UMR 5516, F-42023, SAINT-ETIENNE, France
}

\title{\landSVM: \\Landmarks-based Linear Local Support Vector Machines}
\date{}

\begin{document}

\maketitle

\begin{abstract}
For their ability to capture non-linearities in the data and to scale to large training sets, local Support Vector Machines (SVMs) have received a special attention during the past decade. 
In this paper, we introduce a new local SVM method, called \landSVM, which clusters the input space, carries out dimensionality reduction by projecting the data on landmarks, and jointly learns a linear combination of local  models. Simple and effective, our algorithm is also theoretically well-founded. Using the framework of Uniform Stability, we show that our SVM formulation comes with generalization guarantees on the true risk. 
The experiments based on the simplest configuration of our model (i.e. landmarks randomly selected, linear projection, linear kernel) show that \landSVM is very competitive \wrt the state of the art and opens the door to new exciting  lines of research.
\end{abstract}
\section{Introduction}

One of the most famous and commonly used Machine Learning techniques for classification are the Support Vector Machines (SVMs)~\cite{cortes1995support}. This popularity is due to their robustness, simplicity, efficiency (even in non linear scenarios by means of the kernel trick) as well as their theoretical foundations via generalization guarantees. %It allows to learn class-separator large margin hyperplanes and it can be extended to solve non-linear problems by means of the kernel trick.

Despite those nice properties, SVMs may face some drawbacks: Kernel SVMs are known to be expensive in terms of time complexity and memory usage when the number of training examples is large, both at training and at testing time.
For training, the full Gram matrix needs to be evaluated (i.e., compute and store all pairwise training sample similarities), and then inverted. % is it really an inversion?
For testing, the time complexity depends on the number of support vectors which typically grows linearly with the number of training instances~\cite{steinwart2003sparseness}.
Therefore, Kernel SVMs have been shown not to scale well to very large data sets. 

Over the years, several methods have been proposed to speed up SVMs, for instance by reducing the size of the training set~\cite{bakir2005breaking}, or by making use of stochastic optimization~\cite{bordes2009sgd} or by solving an alternative formulation of the orginal SVM problem~\cite{Joachims:2006}.
On the other hand, \textit{locally-linear learning approaches} have been shown to be most appealing in terms of training time, testing time and accuracy.
They are effective for data sets that present multi-modalities and/or non-linearities because they are able to capture the local characteristics of the space.
They are also computationally efficient as they learn only linear classifiers (for which efficient solvers exist) and, at testing time, are independent of the number of support vectors.
One drawback of such techniques is that they may be subject to local over-fitting. % what fact do we have about this overfiting? do we solve it?
We can distinguish \textit{two main families} of local SVM approaches: the ones that locally learn combinations of a set of learned linear SVMs as in~\cite{ladicky2011locally,fornoni2013multiclass}, and those which partition the input space and learn a local model per region~\cite{gu2013clustered,fu2010mixing}.

% LLSVM:
% - binary (they say this can be generalized), no theoretical guarantees
% - SGD for online, or else QP solving
% LM3SVM: multiclass, better, no expressed guarantees, slightly slow (high for dimensional data)
% -  Our model is trained using a CCCP procedure (Smola et al., 2005), where each of the CCCP subproblems is solved using SGD

Methods from the first category estimate local combinations of linear SVMs and make the assumption that the input data are lying on a manifold along which the linear classifier evolves smoothly.
In~\cite{ladicky2011locally}, the manifold is approximated by selecting some anchor points (using k-means) and learning one local model per anchor point.
Each training point is then expressed, using a local coding scheme, as a linear combination of its closest anchor points.
The local coding ensure that each point is influenced by a limited number of models and thus that the learning is efficient.
Although effective, one drawback of this approach is that the influence of the anchor points is defined by a fixed neighborhood that has to be manually set.
A latent SVM formulation is used in~\cite{fornoni2013multiclass} where the authors follow the principle of~\cite{ladicky2011locally} but extend it to a multiclass setting and replace the local coding scheme by latent coordinates that are estimated jointly with the parameters of the linear models.
% Speed, problems, issues, limitations ?

% Clustered SVM:
% - has generalization bound using Rademacher complexity (and mention that k-means has theoretical properties (it's VC dim) that are used)
% - generally better than k-means+svm(baseline) and MSVM? and LLSVM (other family)

Methods from the second family, such as clustered SVMs~\cite{gu2013clustered}, first partition the input space, typically using k-means, and then learn a linear model in each region.
To cope with overfitting, clustered SVMs use a hierarchical regularization: the vector of parameters of each local model is constrained to be close to an unknown, shared, global parameter vector.
Using only linear combinations of linear models, this approach is shown to be fast while yielding accuracies better or comparable to the state of the art.
A limitation of this approach comes from the assumption that the same global model is both meaningful and sufficient to regularize the different local SVMs.
An advantage, over the first family of methods, is that a generalization error bound can be derived using the Rademacher complexity.

In this paper, we introduce a new local SVM method, called \landSVM, that targets computational efficiency while having provable theoretical guarantees.
Our method clusters the input space, carries out dimensionality reduction by projecting all points on selected landmarks, and learns interdependent linear combinations of linear models.
As such, our method lies in between the two families of local approaches presented above without suffering from the mentioned drawbacks.
On one hand, it can be seen as learning a set of linear models that are combined following a latent space (linear or not) induced by the set of landmarks that is common to all clusters.
On the other hand, the proposed method can be seen as clustering the input space and learning, in a projected space, a set of interacting linear models.

Using the framework of the Uniform Stability~\cite{bousquet2002stability}, we prove that our algorithm is stable \wrt changes in the training set allowing us to derive a tight generalization bound on the true risk.
It is worth noticing that our algorithm, which can be seen as a generalization of the standard SVM formulation, is configurable and offers many points for improvement: clustering algorithm, regularization terms, landmark selection method, projection function, etc.
However, while many variations can be imagined, our early experiments surprisingly show that the ``default'' choices (k-means clustering, random landmarks, linear projection, linear kernel) already yield an algorithm that is competitive with the state of the art while extremely fast and scalable.

The remainder of this paper is organized as follows: in Section~\ref{sec:method} we give a mathematical formulation of \landSVM and we analyze its complexity; in Section~\ref{sec:stability} we theoretically study our method through the Uniform Stability framework; in Section~\ref{sec:expe} we 
empirically study the impact of different configurations (e.g. the number of landmarks and the number of clusters) and compare our method to state-of-the-art local SVM-based methods; finally in Section~\ref{sec:conclpersp} we present some exciting perspectives of this work.

\section{Soft-margin Landmarks-based Linear Local SMVs}
\label{sec:method}

Our method, called \landSVM, consists in partitioning the input space into $K$ clusters and learning $K$ corresponding (linear) models that interact in a single optimization problem.
The interactions come from a projection on a set of landmarks $\mathcal{L}$ that is common for all clusters and from the formulation of a unique linear problem with a single bias parameter.
It is worth noticing that a standard SVM is a particular case of our approach for $K=1$ and specifically chosen landmarks.

\subsection{Notations and Optimization Problem}

Let $\mathcal{X} \subseteq \mathbb{R}^n$ be the input space, $\mathcal{Y} = \{-1,1\}$ the output space and $\left\{C_k\right\}_{k=1}^K$ a partition of $\mathcal{X}$. We consider a training sample $\mathcal{S} = \left\{z_i=(x_i,y_i,k_i)\right\}_{i=1}^{m}$ of $m$ i.i.d. instances $z_i \in \mathcal{X} \times \mathcal{Y} \times \{1,..,K\}$ (such that $x_i \in C_{k_i}$) drawn from an unknown distribution $\mathcal{D}$.
Moreover, we  denote $\mathcal{L} = \left\{l_p\right\}_{p=1}^L \in \mathcal{X}^L$, a set of $L$ landmarks of the input space. 
The objective function of \landSVM is defined as follows:

$$ F(f) = \frac{1}{2} \norm{f}^2 + \frac{c}{m} \sumi \ell(f,z_i)$$

where $\ell(f,z) = \max(0,1-y f(x,k)) $ is the hinge loss and $f: \mathcal{X} \times \{1,..,K\} \to \mathbb{R}$ the function

$$ f(x,k) = \sump \theta_{kp} \mu(x,l_p) + b$$

which is  used for prediction with: $\hat{y} = sign(f(x,k))$.

Note that $\theta \in \mathbb{R}^{K \times L}$ is a matrix of weights expressing the influence of each landmark $p$ for a given cluster $C_k$. Doing so, we are supposing that the problem is linear in the space created by  both clusters and landmarks. Thus, we learn a vector of weights per cluster but a unique offset $b$. In Appendix~\ref{an:graphicalmodels}, we provide a visualization of the dependencies between the problem variables that can help understand our method.

Another way to see our method is as learning an SVM classifier in a projected space defined by the selected landmarks $\mathcal{L}$ and by a score function $\mu: \mathcal{X}^2 \to \mathbb{R}$ between points of the input space:

$$ f(x,k) = \theta_{k.} \mul{x}^T + b$$

where $\mul{.} = [\mu(.,l_1),...,\mu(.,l_L)]$ is a projection from the input space $\mathcal{X}$ to the landmark space $\mathcal{H}$.
Therefore, the clusters allow to capture the non-linearities of the space while the landmarks help to control the size of the input space. %: it is possible to reduce the number of dimensions by eliminating the redundant features, or to add new information by creating new ones.
Additionally, projecting on the landmarks acts as a regularization: as the landmarks are chosen without considering their class and the projection of an instance uses all the landmarks and not only those belonging to its partition, the risk of overfitting is reduced. Therefore, unlike clustered SVMs~\cite{gu2013clustered}, we don't need to learn an additional global model to regularize the local ones. 

As previously mentioned, our method is a generalization of standard SVMs: it is similar to SVMs when $K = 1$ and the set of landmarks $\mathcal{L}$ forms a basis of the input space $\mathcal{X}$, and fully equivalent if this basis is also orthonormal.

A Soft-Margin version of our optimization problem can be written as follows:
$$ \argmin_{\theta,b,\xi} \frac{1}{2} \normf{\theta}^2 + \frac{c}{m} \sumi \xi_i$$
$$s.t. \: y_i \left(\theta_{k_{i.}} \mul{x_i}^T + b \right) \geq 1- \xi_i \:\: \forall i=1..m$$
$$\xi_i \geq 0 \:\: \forall i=1..m$$
which boils down to maximizing the margin between the class hyperplanes while minimizing the average classification error.

\label{sec:where:cankerneltrick}
The previous problem is defined for the linear case but it can be easily rewritten for kernel SVMs considering that there exists an unknown mapping $\phi$ from $\mathcal{H}$ to $\mathcal{Z}$, a space with potentially infinite dimensions, such that $\phi(\mul{x_i})^T\phi(\mul{x_j}) = kernel(\mul{x_i},\mul{x_j})$ and assuming that $\mathcal{H}$ is a Reproducing Kernel Hilbert space~\cite{aronszajn1950theory}.
The kernelized problem can be solved using its dual Lagrangian formulation (see App.~\ref{an:dual}).
However, the advantages of locally learning non-linear SVMs are limited, as our approach already captures non-linearity and has lower complexity compared to kernel SVMs. 

By solving the problem in its dual form, we can also study the relation between the learned model and the support vectors.
The parameters are computed as follows (with $\left\{z_a = (x_a,y_a,k_a)\right\}_{i=1}^A$ the set of $A$ support vectors and $\alpha_a$ the dual value of $z_a$):

$$\theta_{kp} = \sum_{a=1 | k_{a}=k}^A \alpha_a y_a \mu(x_a,l_p)$$
$$b = \frac{1}{A}\sum_{a=1}^A (y_a - \theta_{k_{a.}} \mul{x_a}) $$
which means that the weight $\theta_{kp}$ for a cluster $k$ and a landmark $l_p$ depends on the support vectors of that particular cluster and on their similarities with $l_p$, while the parameter $b$ is computed using the global information obtained from all the support vectors. 

\subsection{Computational Analysis}
As previously mentioned, the main drawback of Kernel SVMs is their inability to scale with large datasets. As a matter of fact, their training complexity is cubic with the number of instances and their testing and memory complexities depend on the number of support vectors which is $O(m)$~\cite{steinwart2003sparseness}.

The proposed approach, if solved in its primal (e.g. using~\cite{REF08a}), has a complexity close to linear SVMs while capturing non-linearities. In Table.~\ref{tab:complexity}, we compare \landSVM with standard Linear-SVMs and RBF-SVMs in terms of training, testing and memory (for storing the learned model) complexities.
For \landSVM we consider the default configuration (that is also used in the experiments of Sec.~\ref{sec:expe}): clustering with k-means, random selection of landmarks and projection with the dot product. 

The training complexity of our method could also be improved by using recent optimization techniques proved to reduce the training time, such as~\cite{bakir2005breaking,bordes2009sgd}.

% training: KLm
% clustering: Km
% selecting landmarks: L -> n
% projecting: nLm -> n^2m
% testing: L*n+L
% memory: KL+Ln
\begin{table}
    \caption{Computational comparison, with $K$: the number of clusters ($K \ll m$), $L$: the number of landmarks ($O(n)$), with $n$: the number of features, and $m$: the number of training instances.}
    \label{tab:complexity}
    \centering
    \scalebox{1}{
    \begin{tabular}{ | c | c | c | c | }
      \hline
      & \textbf{Training Time} & \textbf{Testing Time} & \textbf{Memory Usage} \\ \hline
      \textbf{Linear-SVM} & $O(nm)$ & $O(n)$ & $O(n)$ \\ \hline
      \textbf{RBF-SVM} & $O(m^3)$ & $O(nm)$  & $O(nm)$ \\ \hline
      \textbf{\landSVM} & $O(KLm+Lnm)$ & $O(Ln)$ & $O(KL+Ln)$ \\ \hline
    \end{tabular}}
\label{fig:complexity}
\end{table}

\section{Theoretical Results}
\label{sec:stability}

In this section, we present a generalization bound on the true risk induced by our algorithm using the theoretical framework of the Uniform Stability~\cite{bousquet2002stability}.
We will see that this theoretical analysis gives some insights about the number of landmarks to select in practice. 

\subsection{\landSVM's uniform stability}

The idea of Uniform Stability is to check if an algorithm produces similar solutions from datasets that are slightly different. Let $S$ be the original dataset and $S^i$ the set obtained after having replaced the $i^{th}$ example of $S$ by a new sample $z_i'$ drawn according to ${\cal D}$.
We will say that an algorithm is uniformly stable if the difference between the loss suffered (on a new instance) by the hypothesis $f$ learned from $S$ and the loss suffered by the hypothesis $f^i$ learned from $S^i$ converges in $O(\frac 1m)$.

For the following analysis, we introduce a new notation that allows us to simplify the derivations. We rewrite

$$ f(x,k) = \theta \, \mul{x}^T $$

with $\theta = [\theta_{0.},...,\theta_{k.},...,\theta_{K.},b]$ and $\mul{x} = [\bm{0},...,\mul{x},\bm{0},...,\bm{0},1]$ (that implicitly depends on $k$) both of size $KL+1$ and

$$ F(f) = \frac{1}{2} \norm{\theta}^2 + \frac{c}{m}\sumi \max(0,1-y_i (\theta \mul{x_i}^T)).$$

\begin{defn}{(\bf{Uniform Stability})}
    A learning algorithm A has uniform stability $2 \frac{\beta}{m}$ \wrt the loss function $\ell$ with $\beta \in \mathbb{R}^{+}$ if

    $$ \sup_{z \sim \mathcal{D}} \abs{\ell(f,z) - \ell(f^{i},z)} \leq 2 \frac{\beta}{m} .$$ 
\end{defn}

The uniform stability is directly implied if 

$$ \forall z \in \mathcal{D}, \;\; \abs{\ell(f,z) - \ell(f^{\setminus i},z)} \leq \frac{\beta}{m}$$

where $f^{\setminus i}$ is the hypothesis learned on $S^{\setminus i}$, the set $S$ without the $i^{th}$ instance $z_i$, which follows from

$$\abs{\ell(f^{i},z) - \ell(f,z)} \leq \abs{\ell(f^{i},z) - \ell(f^{\setminus i},z)} + \abs{\ell(f^{\setminus i},z) - \ell(f,z)}  \leq 2 \abs{\ell(f^{\setminus i},z) - \ell(f,z)}$$

that uses the triangular inequality (at worse, altering a point is like removing a point and adding another one).

In order to study the uniform stability of an algorithm, it is required to prove the  $\sigma$-admissibility of the loss function.

\begin{defn}{(\bf{$\sigma$-admissibility})}
    A loss function $\ell(f,z)$ is $\sigma$-admissible \wrt $f$ if it is convex \wrt its first argument and $\forall f_1,f_2 $ and $\forall z \in \mathcal{Z}$:

    $$ \abs{\ell(f_1,z) - l(f_2,z)} \leq \sigma \abs{f_1(x,k)-f_2(x,k)} .$$
\end{defn}

Following~\cite{bousquet2002stability}, we know that the hinge loss is $1$-admissible.

We can now present the main result about our algorithm \landSVM.

\begin{thm}{\bf{\landSVM Uniform Stability}}
  Assuming that  $ \forall x \in \mathcal{X}, \norm{x} \leq c$,  \landSVM has uniform stability  $ \frac{c L M^2}{m}$,
where $M = \max(c^2,1)$ if  $\mu$ is the dot product and $M = 1$ if $\mu$ uses the RBF kernel.
\end{thm}

\begin{proof}

    As $\ell(f,z)$ is $1$-admissible, $\forall z=(x,y,k) \in \mathcal{Z}$,

    %      \abs{\ell(f^{\setminus i},z) - \ell(f,z)} \! \leq \! \abs{f^{\setminus i}\!(x,k)-f(x,k)} \!=\! \abs{\Delta f (x,k)} \label{testit}$$

    \begin{small}
    \begin{align}
      \abs{\ell(f^{\setminus i},z) - \ell(f,z)} &\leq \abs{f^{\setminus i}\!(x,k)-f(x,k)} = \abs{\Delta f (x,k)} \label{lin:lossdiffabs}
    \end{align}
    \end{small}

    with $ \Delta f = f^{\setminus i} - f$.
    By denoting $\Delta\theta = \theta^{\setminus i} - \theta$, we can derive, $\forall z=(x,y,k) \in \mathcal{Z}$,

    \begin{align}
        \abs{\Delta f (x,k)} &= \abs{\theta^{\setminus i} \mul{x}^T - \theta \mul{x}^T} \nonumber \\
        &= \abs{(\theta^{\setminus i}- \theta)\mul{x}^T} \nonumber \\
        & \leq \normf{\theta^{\setminus i} - \theta} \norm{\mul{x}} \label{lin:cauchy} \\
        & \leq \normf{\Delta\theta} \norm{\mul{x}} \nonumber \\
        & \leq \normf{\Delta\theta} \sqrt{L} \norm{\mul{x}}_\infty \label{lin:inf} \\
        & \leq \normf{\Delta\theta} \sqrt{L} \max_l(\mu(x,l)) \nonumber \\
        & \leq \normf{\Delta\theta} \sqrt{L} M \label{lin:dthetasqlm}
    \end{align}

    Eq.~\eqref{lin:cauchy} is due to the Cauchy-Swartz inequality,
    % Eq.~\eqref{lin:theta} is because $ \norm{\Delta f (.,k_i)} = \norm{\theta^{\setminus i} - \theta}$ ($\Delta f$ is linear in it's first parameter)
    and Eq.~\eqref{lin:inf} is because $ \norm{\mul{x}} \leq \sqrt{L} \norm{\mul{x}}_\infty$ recalling that $\mul{x} \in \mathbb{R}^{(1 \times L)}$.

    The value of $M$ depends on the chosen function $\mu$. For instance, if $\mu$ is the dot product, $M = \max(C^2,1)$ and if it uses the RBF kernel, $M = 1$.

    From Lemma 21 of \cite{bousquet2002stability}:

    $$ 2 \normf{\Delta\theta}^2 \leq \frac{c}{m} \abs{\Delta f(x_i,k_i)}.$$

    Then, by instantiating Eq.~\eqref{lin:dthetasqlm} for $z = z_i$, we get

    $$\normf{\Delta\theta}^2 \leq \frac{c}{2m} \abs{\Delta f(x_i,k_i)} \leq \frac{c}{2m} \normf{\Delta\theta} \sqrt{L} M$$

    and as $\normf{\Delta\theta} > 0$, we obtain

    $$ \normf{\Delta\theta} \leq \frac{c}{2m} \sqrt{L} M $$

    so, from the previous bound on $\abs{\Delta f(x,k)}$, we get

    $$ \forall z=(x,y,k), \;\; \abs{\Delta f(x,k)} \leq \normf{\Delta\theta} \sqrt{L} M \leq \frac{c L M^2}{2m}$$

    which, with Eq.~\eqref{lin:lossdiffabs} gives the $\frac{c L M^2}{m}$ uniform stability.

\end{proof}

Note that the stability of the algorithm depends on the number of selected landmarks. \landSVM is stable only if $L \ll m$, which is not a strict condition considering that, in practice, we select $L = O(n)$ landmarks (with $n$ the size of the input space $\mathcal{X}$) and that, for learning in general, $n \ll m$.

\begin{thm}{\cite{bousquet2002stability}}
Let A be an algorithm with uniform stability $\frac{2\beta}{m}$ \wrt a loss $\ell$ such that $0 \leq \ell(f,z) \leq E$, $\forall z \in \mathcal{Z}$. Then, for any i.i.d. sample $S$ of size $m$ and for any $\delta \in (0,1)$, with probability $1- \delta$:

$$ R_{\mathcal{D}}(f) \leq \hat{R}_{S}(f) + \frac{2\beta}{m} + \big( 4\beta + E \big) \sqrt{\frac{\ln \frac{1}{\delta}}{2m}}$$

where $R_{\mathcal{D}}(f)$ is the true risk and $\hat{R}_{S}(f)$ is the empirical risk on sample $S$. 

\end{thm}

Before deriving the generalization bound, we need to prove that our loss $\ell$ is bounded by a constant $E$ when evaluated at the optimal solution of $F$. Let $f$ be the minimizer of $F$. We deduce that:

\begin{gather}
    F(f) \leq F(\bm{0}) \nonumber \\
    \frac{1}{2} \norm{\theta}^2 + \frac{c}{m} \sumi \max(0,1-y_i (\theta \mul{x_i}^T)) \leq \frac{1}{2} \norm{\bm{0}}^2 + \frac{c}{m} \sumi \max(0,1-y_i (\bm{0} \mul{x_i}^T)) \nonumber \\
    \frac{1}{2} \norm{\theta}^2  \leq c \label{eq:sum} \\ 
    \norm{\theta}^2  \leq 2c \nonumber
\end{gather}

Eq.~\eqref{eq:sum} is because $ \forall a,b,c \in \mathbb{R}^{+}$, $ a + b \leq c $ implies that $ b \leq c $. Thus,

\begin{align}
    \ell(f,z) &= \max(0,1-y \theta \mul{x}^T) \nonumber \\
    & \leq 1 + \abs{\theta \mul{x}^T} \nonumber \\
    & \leq 1 + \norm{\theta} \norm{\mul{x}^T}  \label{eq:cauchy2}\\
    & \leq 1 + 2c \sqrt{L} M = E \nonumber
\end{align}

Eq.~\eqref{eq:cauchy2} comes again from the Cauchy-Swartz inequality.

\begin{cor}
    The generalization bound of \landSVM derived using the Uniform Stability framework is as follows:

    \small{
    $$ R_{\mathcal{D}}(f)\! \leq \!\hat{R}_{S}(f) + \frac{c L M^2}{m} + \left( \frac{2c L M^2}{m} \!+ \!1 \!+\! 2c \sqrt{L} M \! \right)\!\!\sqrt{\frac{\ln \frac{1}{\delta}}{2m}}.$$}

\end{cor}

%In other words, as the size of the sample increases, the true risk tends to be smaller or equal to the empirical risk, which implies that the algorithm generalizes well on unseen data.

\section{Experimental Results}
\label{sec:expe}
In this section, we empirically study the behavior of \landSVM both on synthetic and on real datasets, and for binary and multiclass classification. Specifically, we study the impact of the number of clusters and the number of landmarks on learning, we analyze two different methods for selecting the landmarks and finally we compare our method to the state-of-the-art SVM based techniques.
\landSVM is implemented in Python using the liblinear~\cite{REF08a} library and the multiclass classification is performed through a one-vs-all procedure.

\subsection{Non-linearities}
Here we study the influence of the number of clusters on learning. We compare the performances of standard SVMs (linear or kernelized with a RBF kernel) with those of \landSVM (using the inner product or the RBF projection function) on two toy non-linear distributions: the XOR distribution and the Swiss-roll distribution. Remember that, for our method, even if the function $\mu$ used for projecting the data is the RBF, the learned models are still linear and the learning remains efficient.

For these experiments, we tune the hyper-parameters of each method by grid search with the values $\{10^{-3},10^{-2},10^{-1},1,10,100 \}$ in a 5-fold cross-validation procedure and for the \landSVM, the number of landmarks is arbitrarily fixed to 10 which are randomly selected from the training sample. The instances are clustered using k-means. In Fig.~\ref{fig:xor} and~\ref{fig:swiss} we draw the learned class separators, as well as the training instances (according to their true label) and the support vectors marked by a black point. We report the training and testing accuracies (on training and testing samples of same size) and the number of support vectors.

\begin{figure}[h!]
  \centering
    \includegraphics[width=0.32\textwidth]{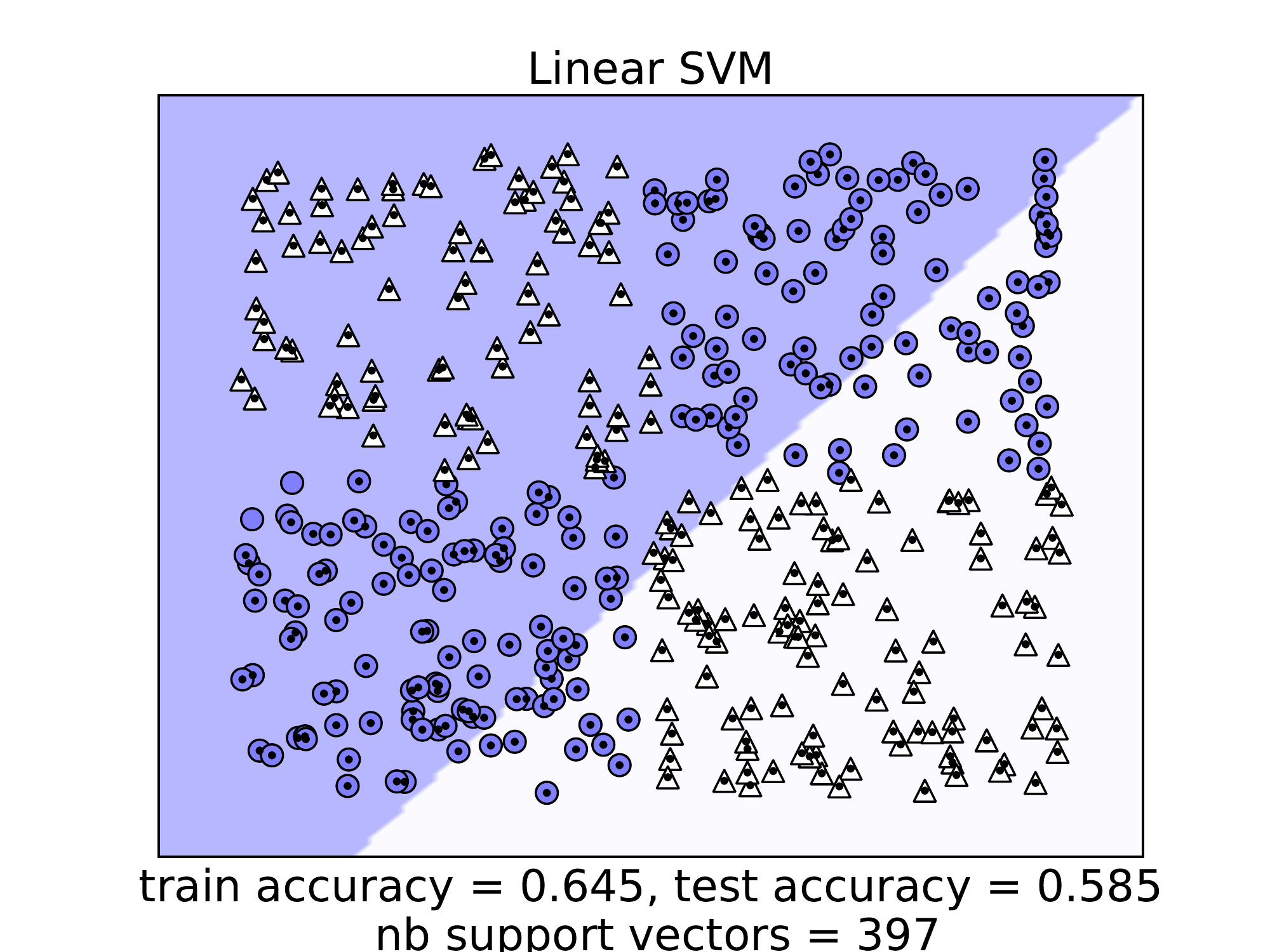}
    \includegraphics[width=0.32\textwidth]{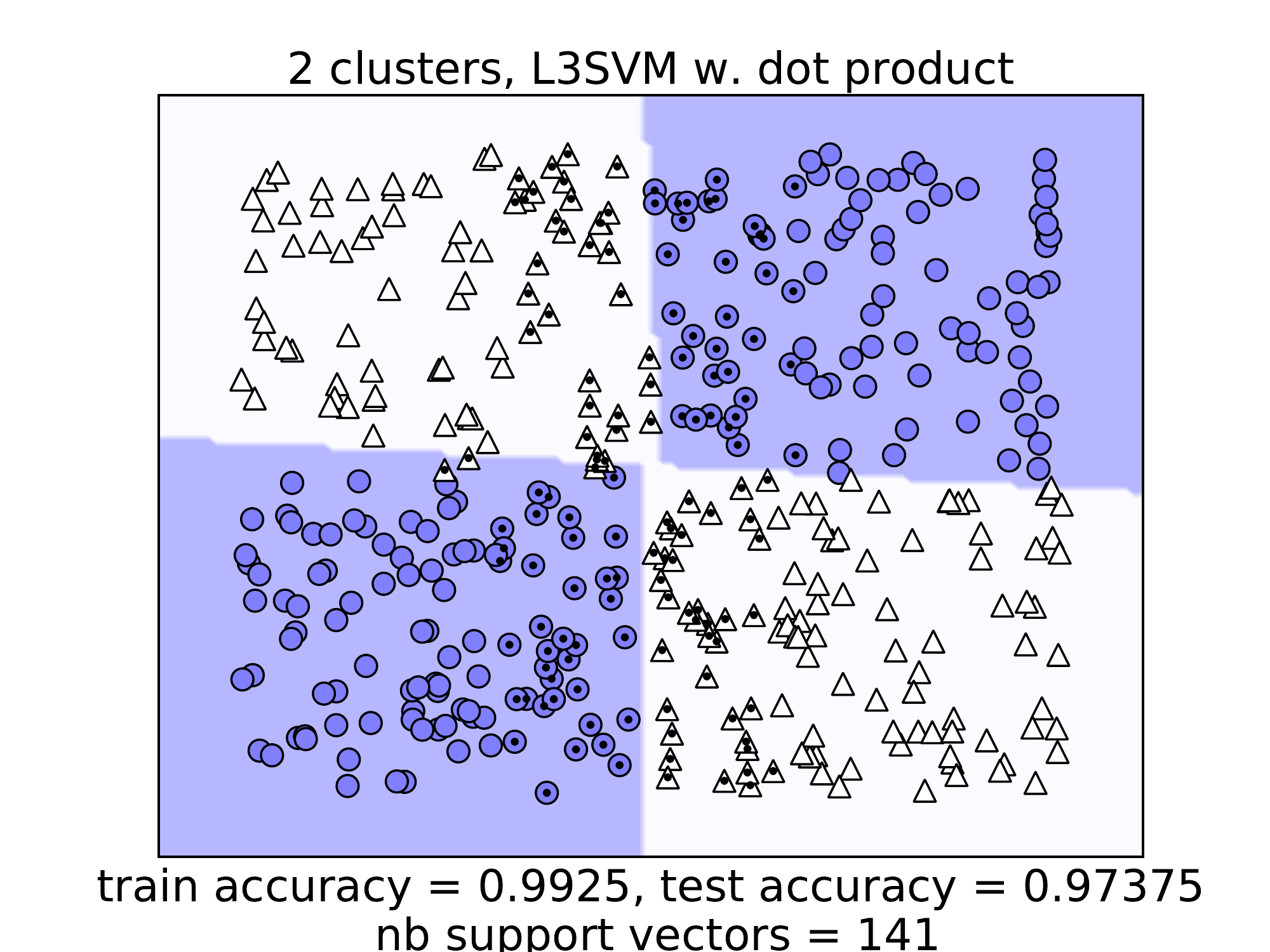}
    \includegraphics[width=0.32\textwidth]{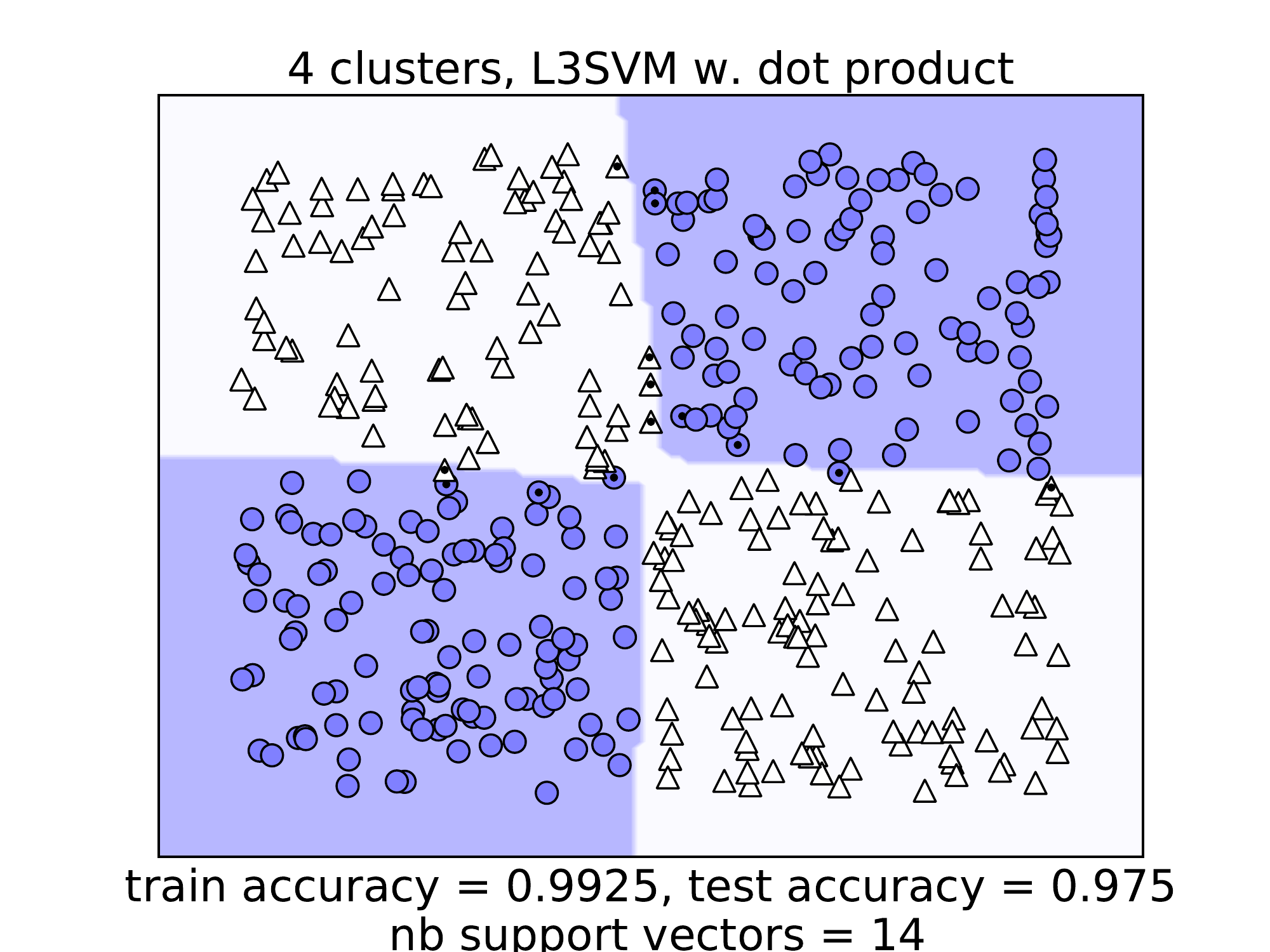}\\
    \includegraphics[width=0.32\textwidth]{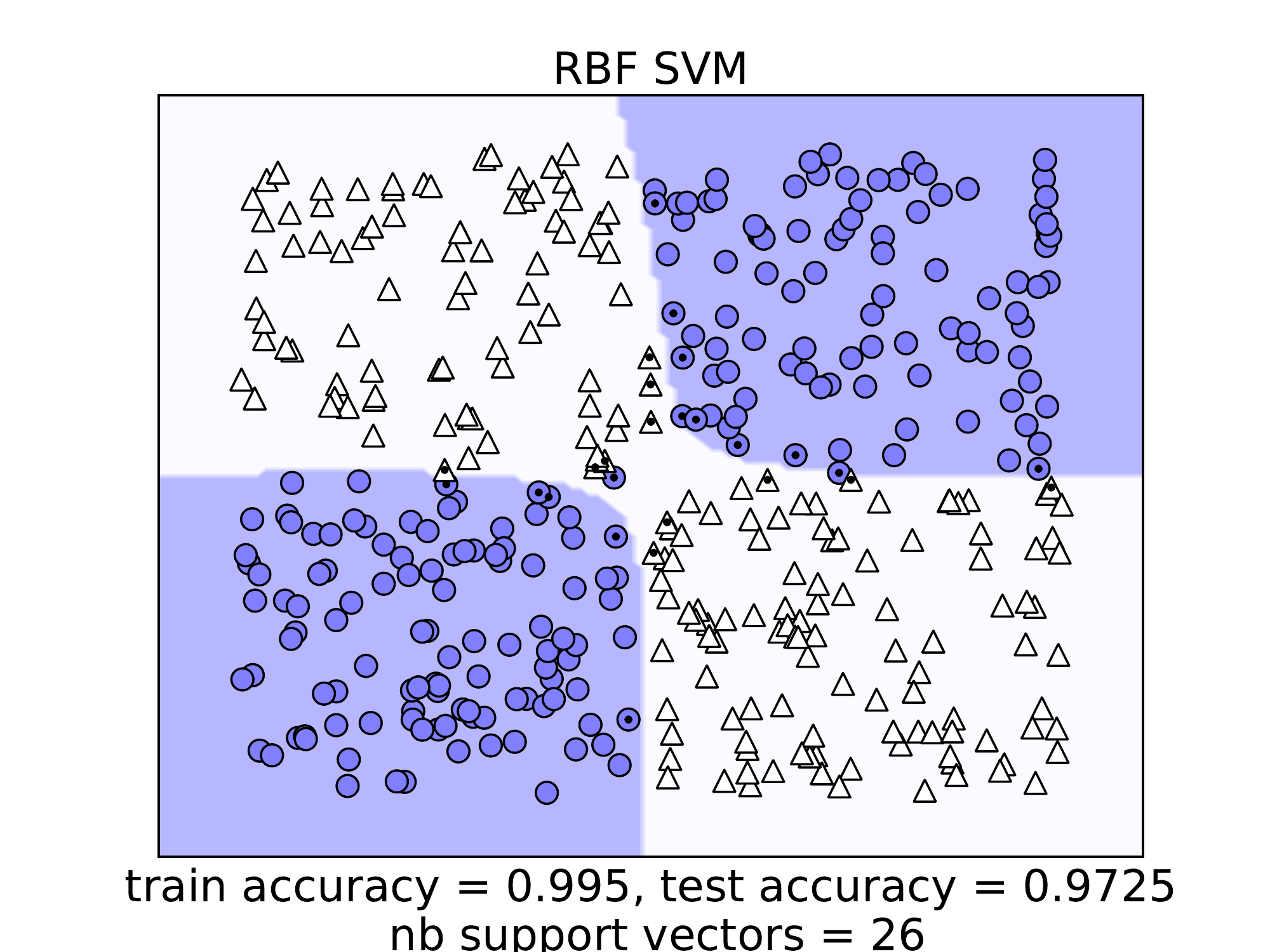}    
    \includegraphics[width=0.32\textwidth]{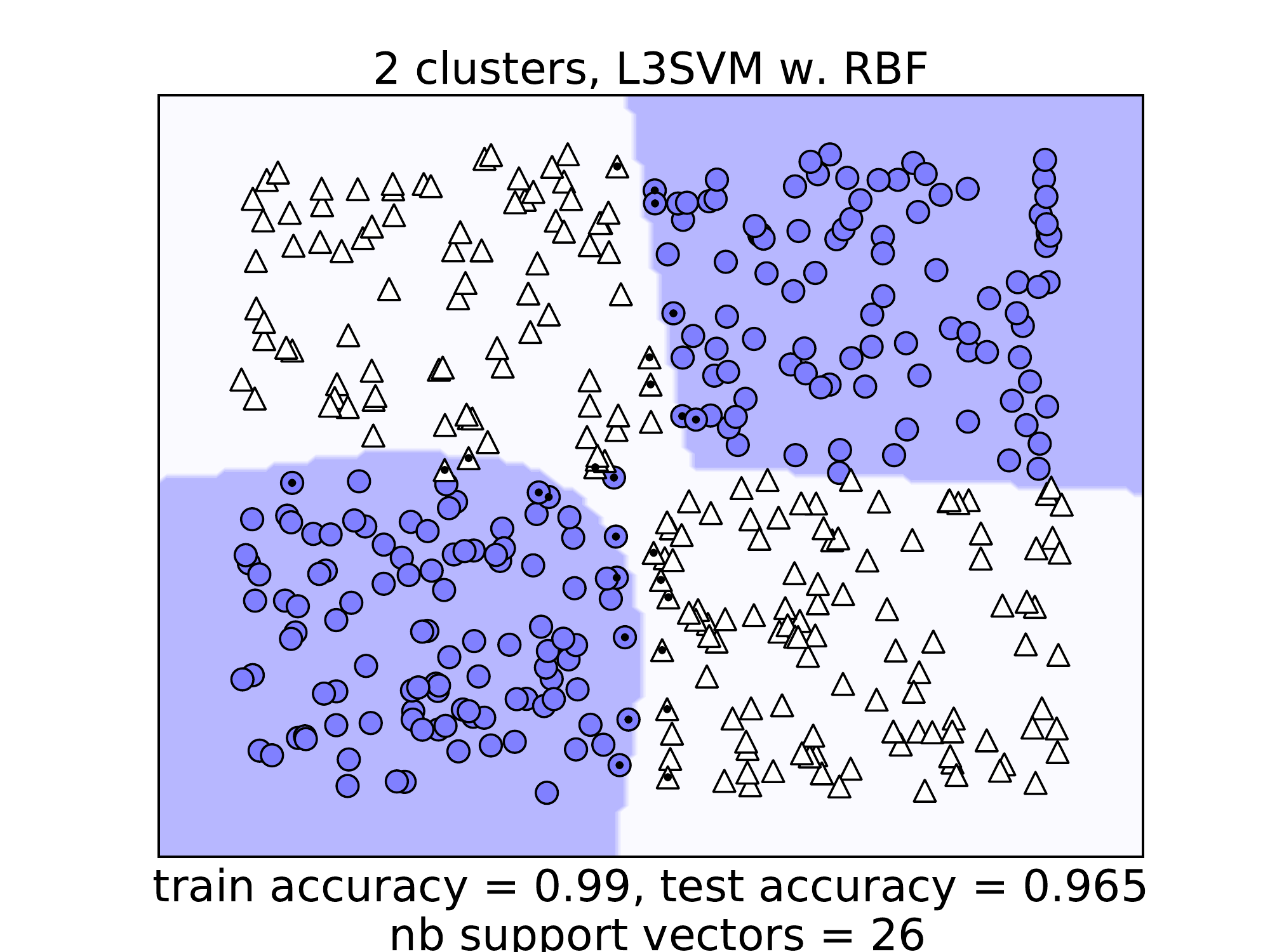}
    \includegraphics[width=0.32\textwidth]{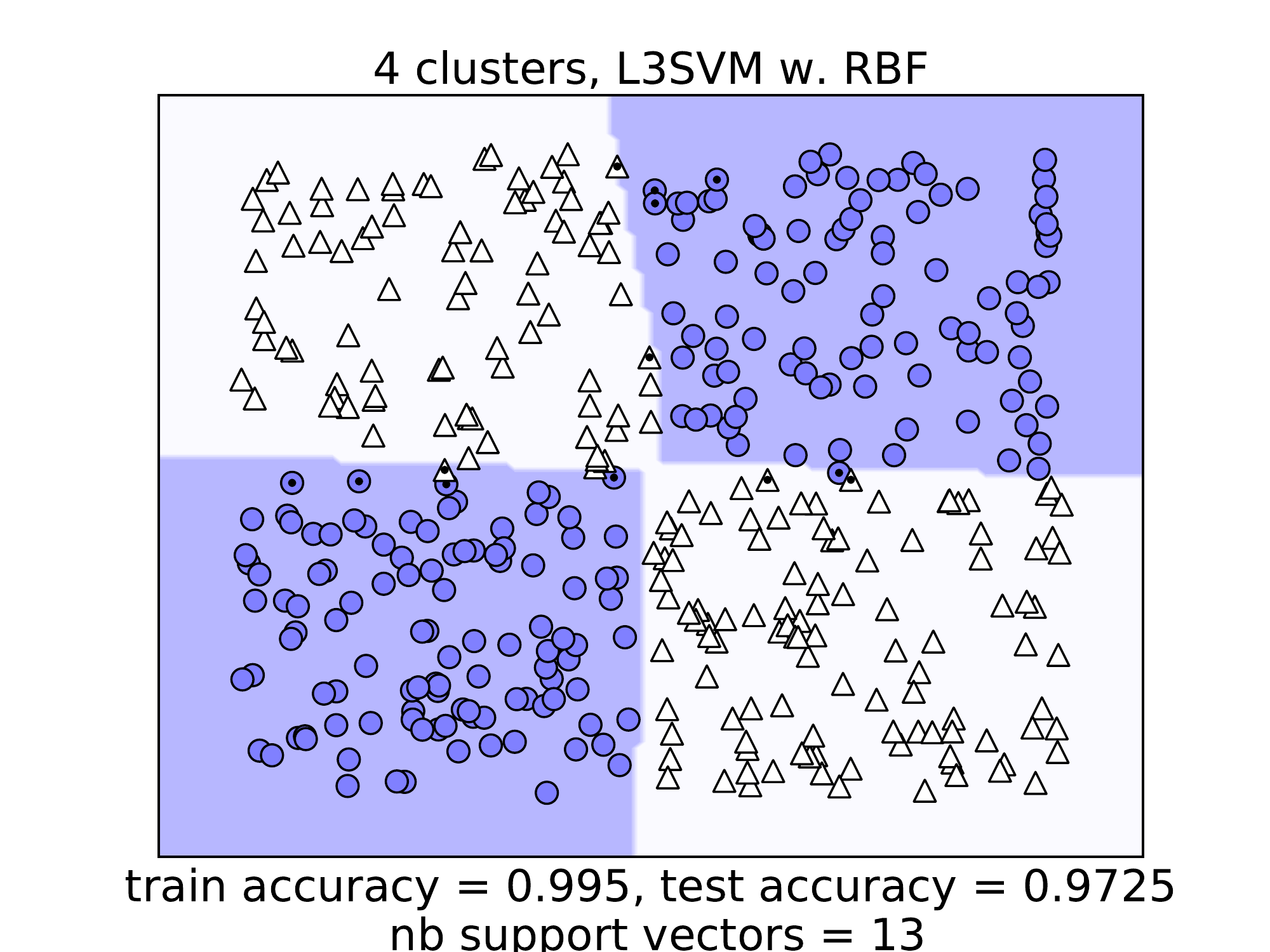}\\
  \caption{\it{2D-XOR distribution}: 400 training instances.}
  \label{fig:xor}
\end{figure}

\paragraph{XOR distribution, Fig.~\ref{fig:xor}} We generated a synthetic XOR distribution by drawing instances uniformly over a 2D-space and assigning to each instance the label $+1$ (resp. $-1$) if its coordinates have the same sign (resp. different signs). As expected, the linear SVM is not able to separate the two classes, while the RBF SVM captures the non-linearities of the space. We notice that the performances of a \landSVM are comparable to the RBF SVM in terms of accuracy and number of support vectors already with 2 clusters and that with 4 clusters we achieve the best results. Moreover the learned class regions are similar to the theoretical ones.

\begin{figure}[h!]
  \centering
    \includegraphics[width=0.32\textwidth]{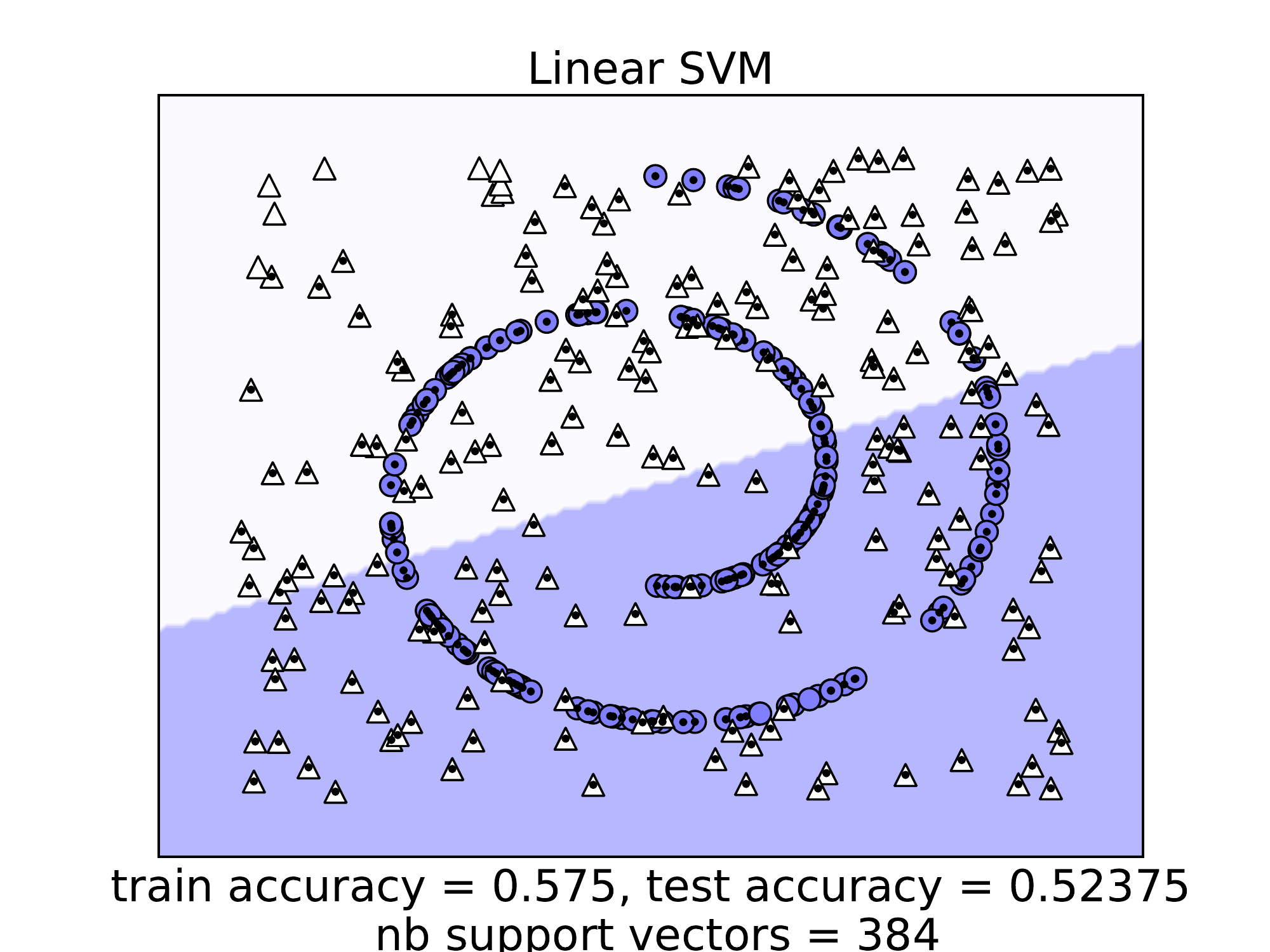}
    \includegraphics[width=0.32\textwidth]{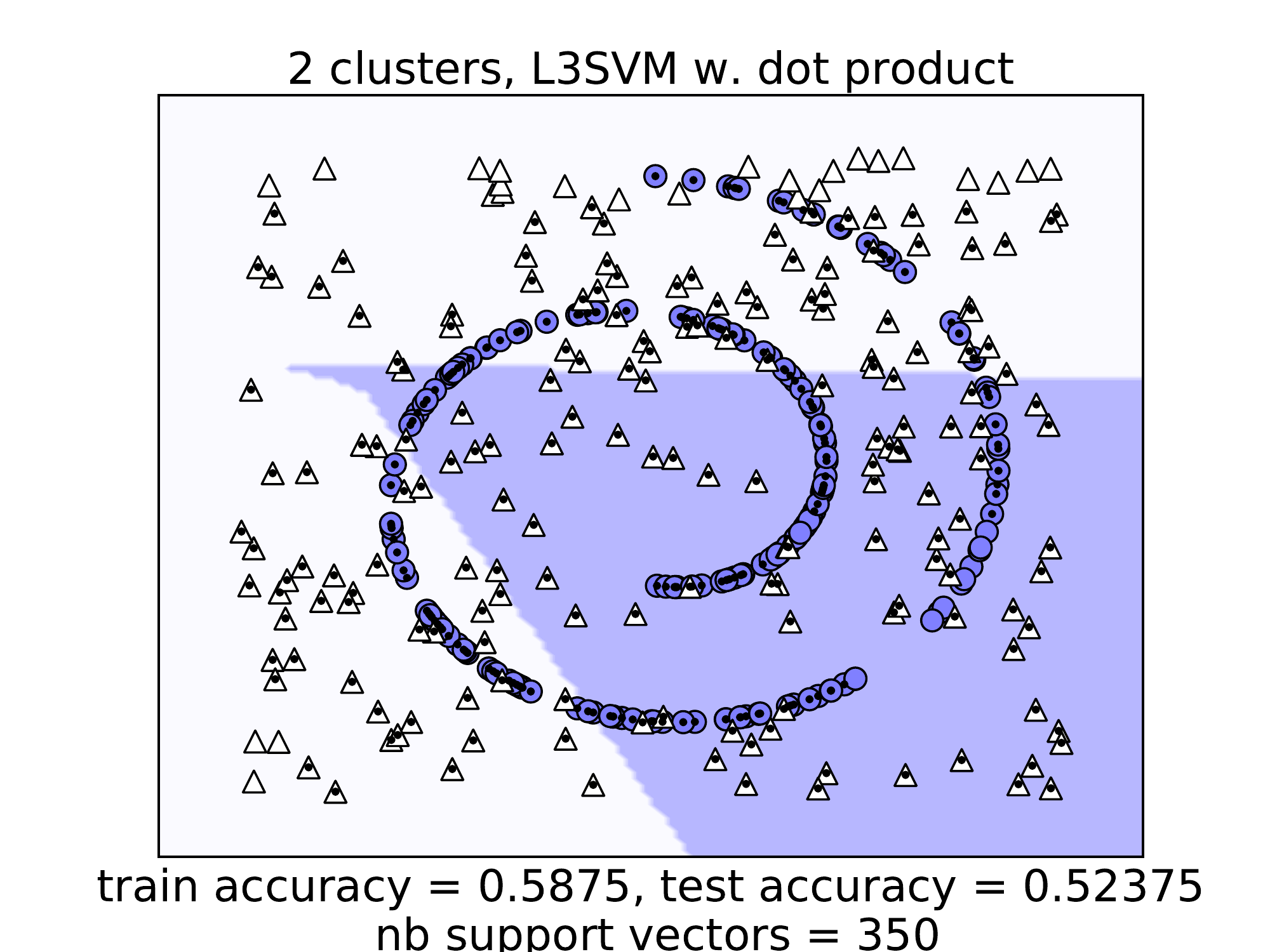}
    \includegraphics[width=0.32\textwidth]{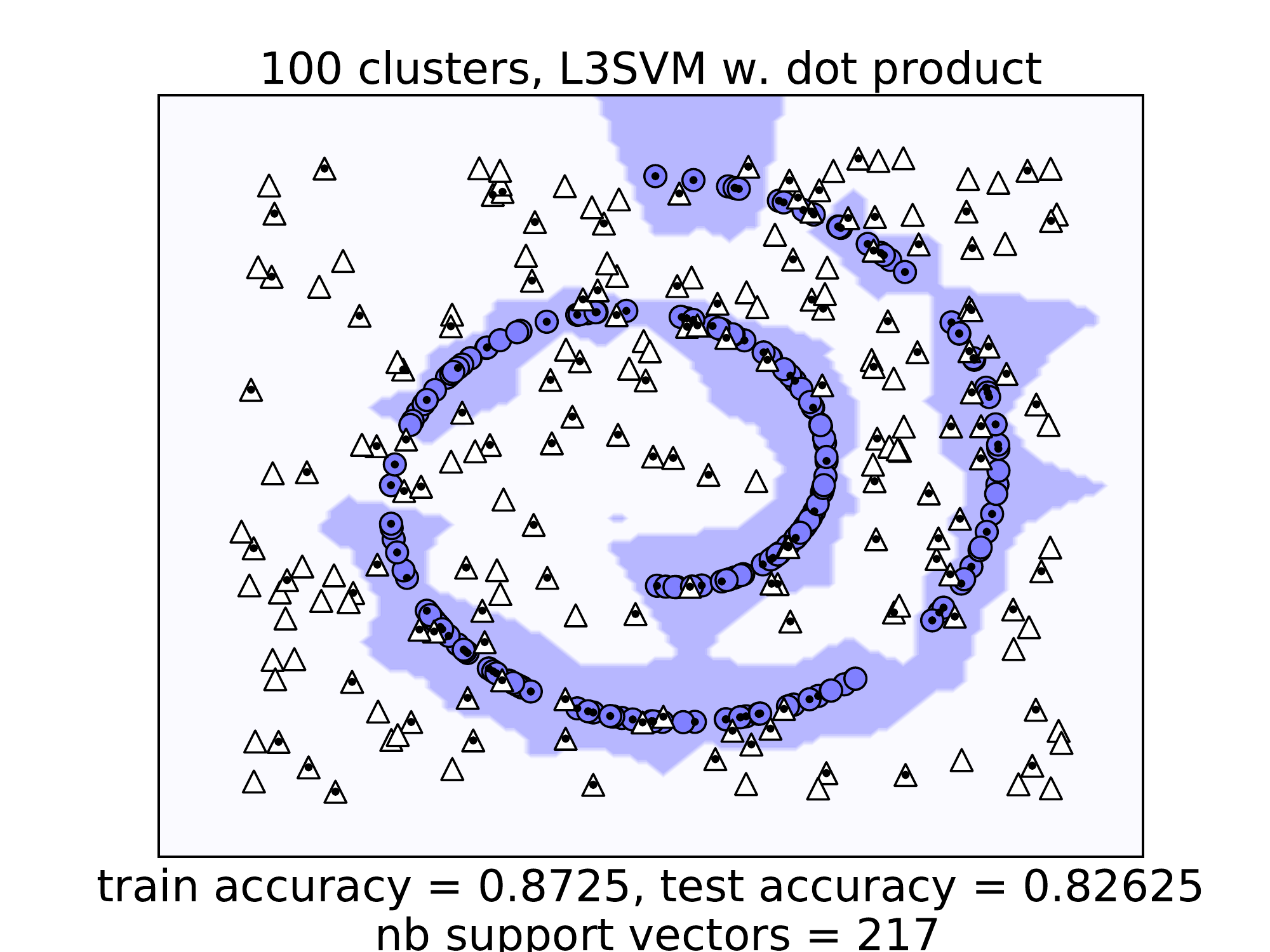}\\
    \includegraphics[width=0.32\textwidth]{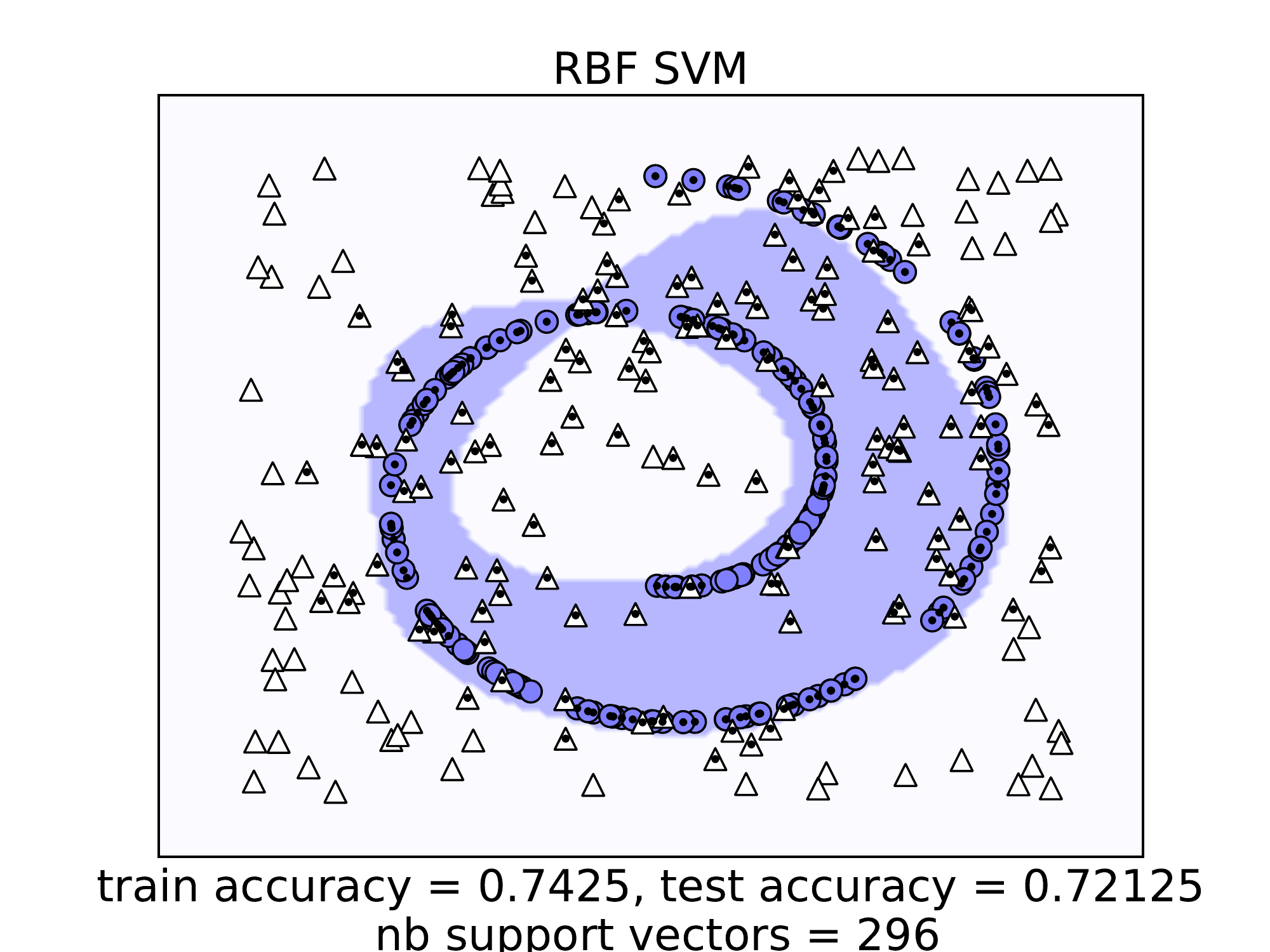}    
    \includegraphics[width=0.32\textwidth]{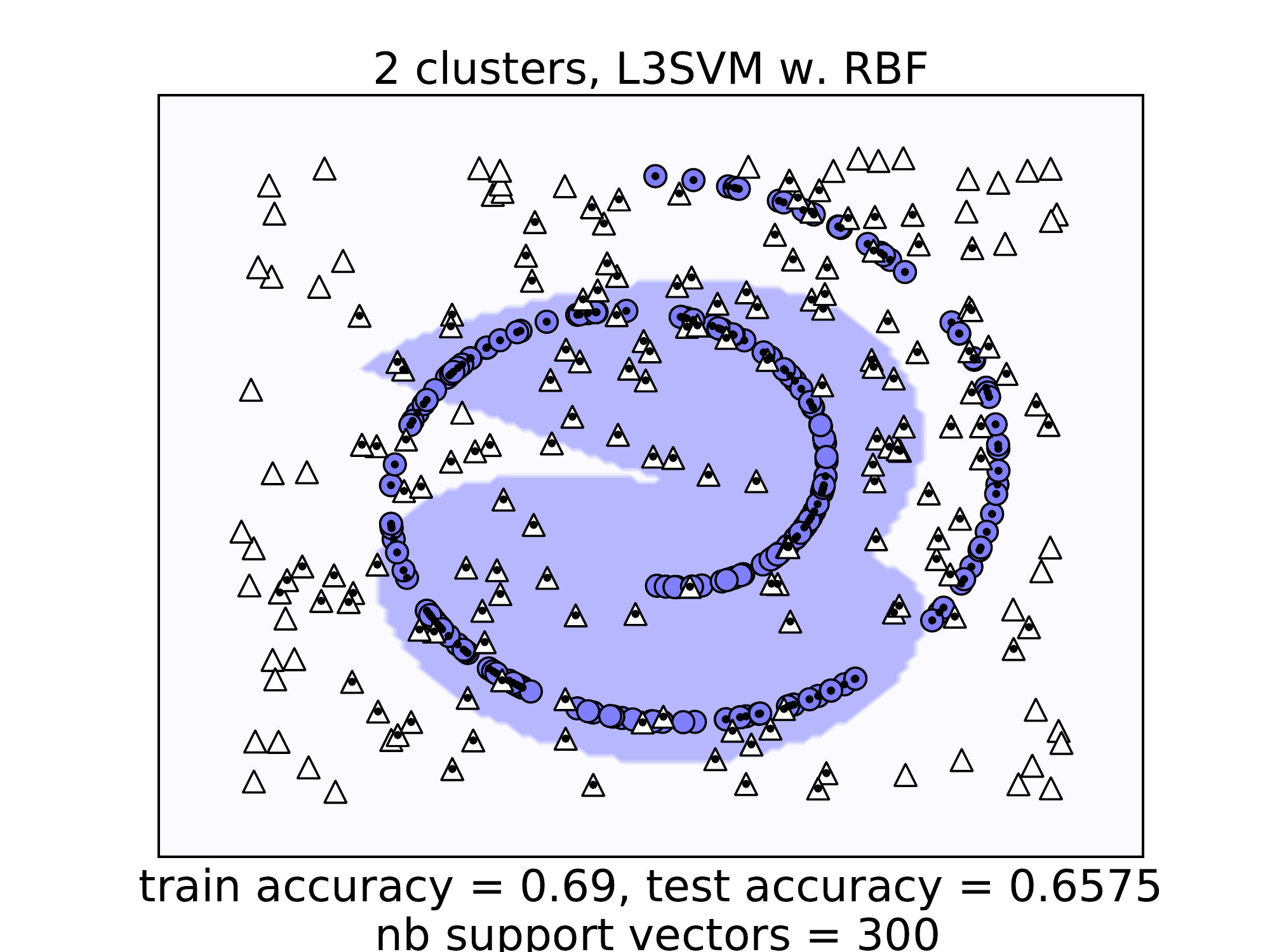}    
    \includegraphics[width=0.32\textwidth]{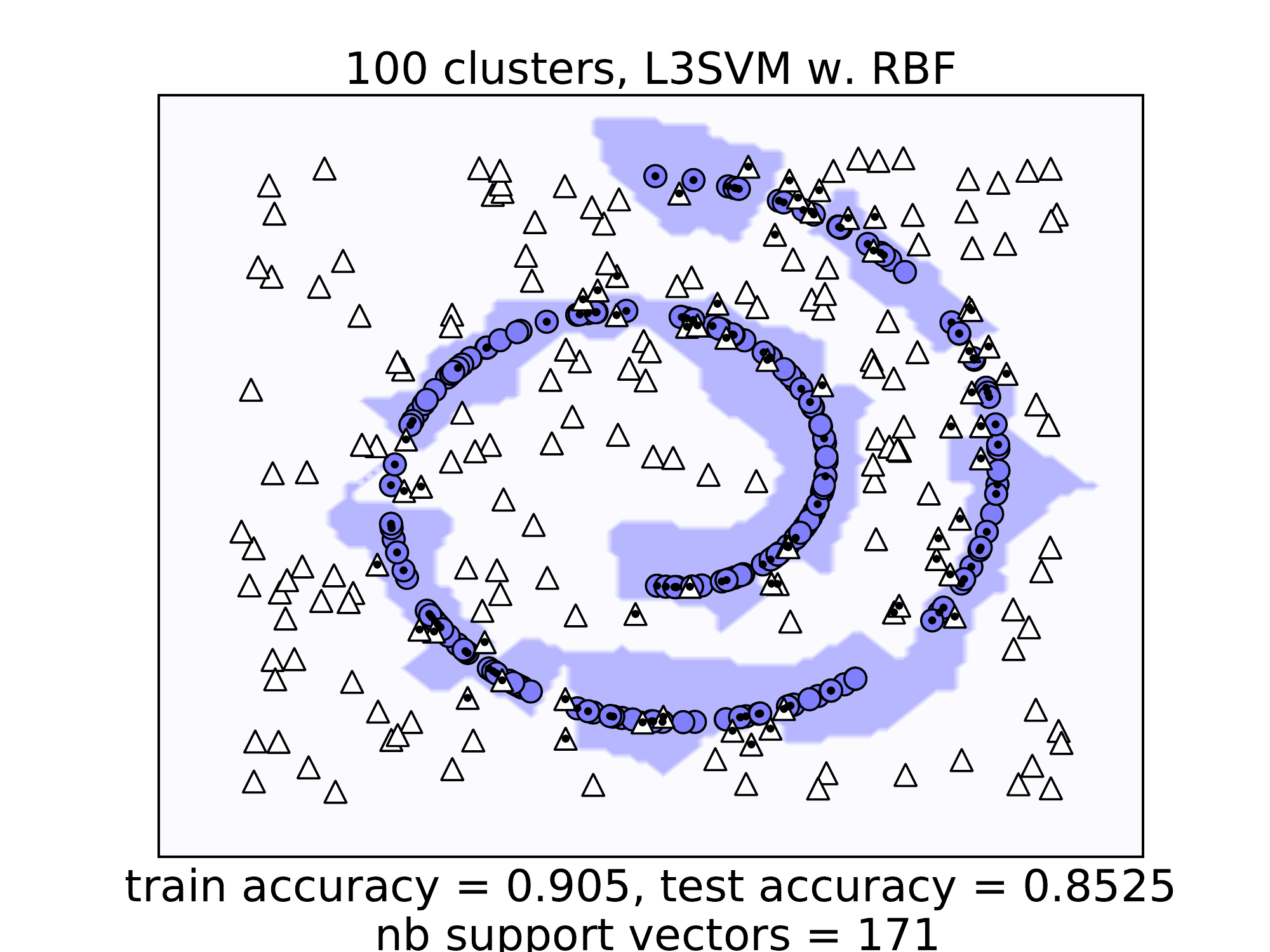}\\
  \caption{\it{2D-Swiss-roll distribution}: 400 training instances, balanced classes.}
  \label{fig:swiss}
\end{figure}

\paragraph{Swiss-roll distribution, Fig.~\ref{fig:swiss}} The problem consists in separating a Swiss-roll distribution (the first class) from a uniform one (the second class). Unlike the XOR distribution, in this case 2 clusters are not enough to capture the non-linearities of the space, but with 100 clusters we obtain better performances than the ones of a Kernelized SVM. 

Notice that, in both experiments, as the number of clusters increases, the difference in accuracy between a \landSVM with a very fast inner product and a \landSVM with a RBF projection function is irrelevant.
Our method is then able to capture the non-linearities of the space as well as a non-linear SVM.
Note that the number of clusters depends, above all, on the nature of the input space.

\subsection{Choice of $L$}
\begin{figure*}[h!]
  \begin{subfigure}{\textwidth}
    \centering
      \includegraphics[width=0.30\textwidth]{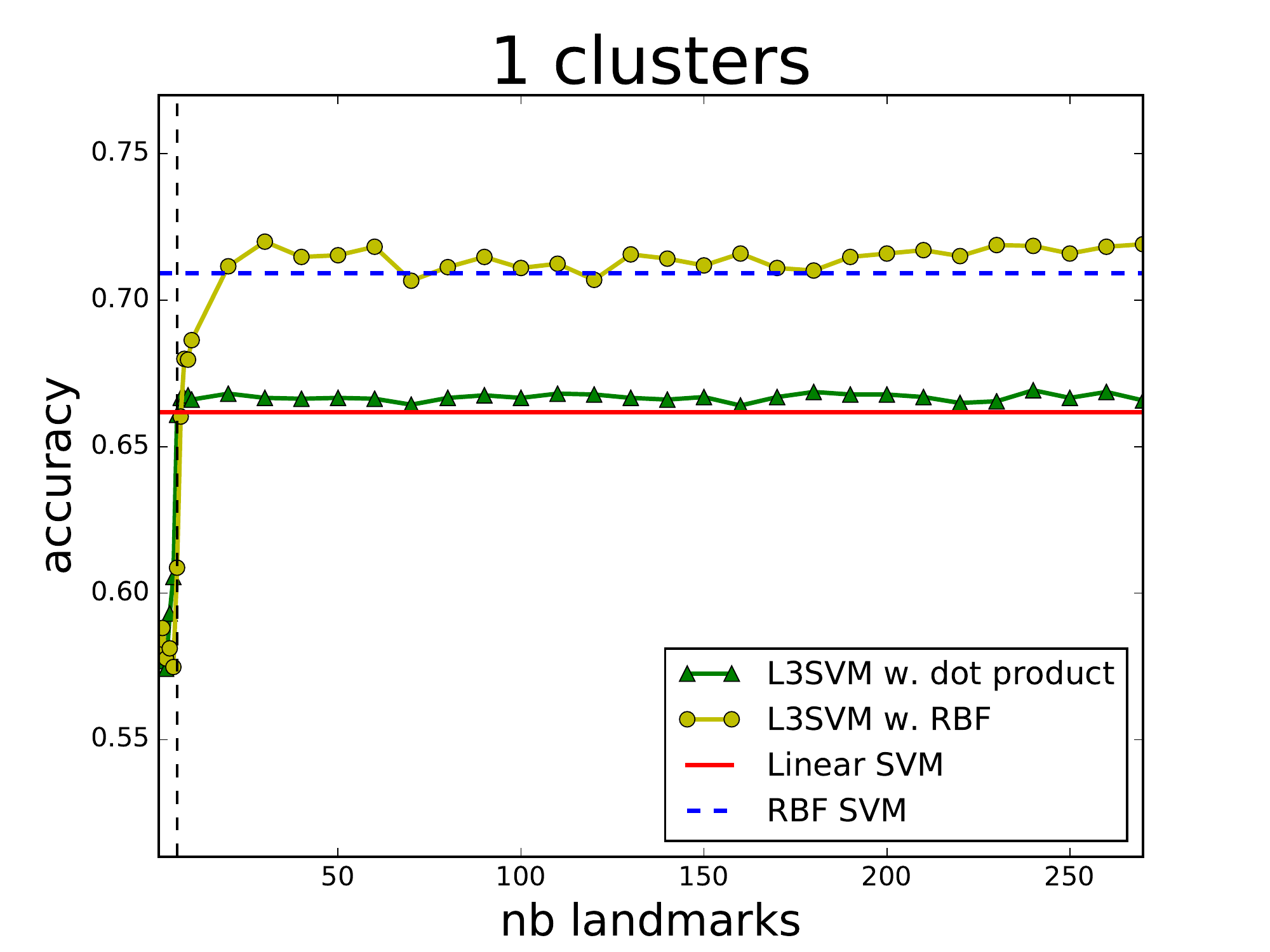}
      \includegraphics[width=0.30\textwidth]{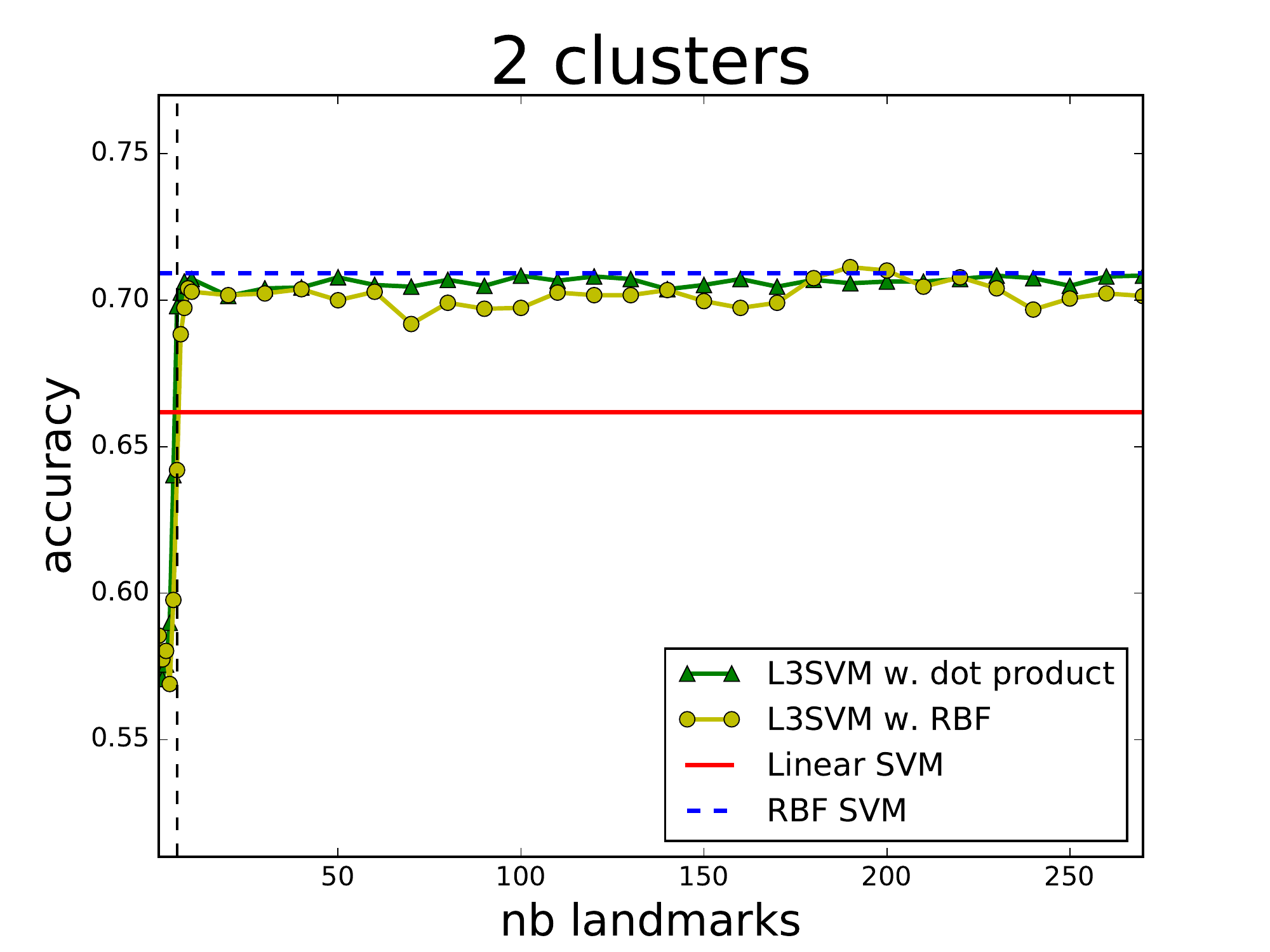}
      \includegraphics[width=0.30\textwidth]{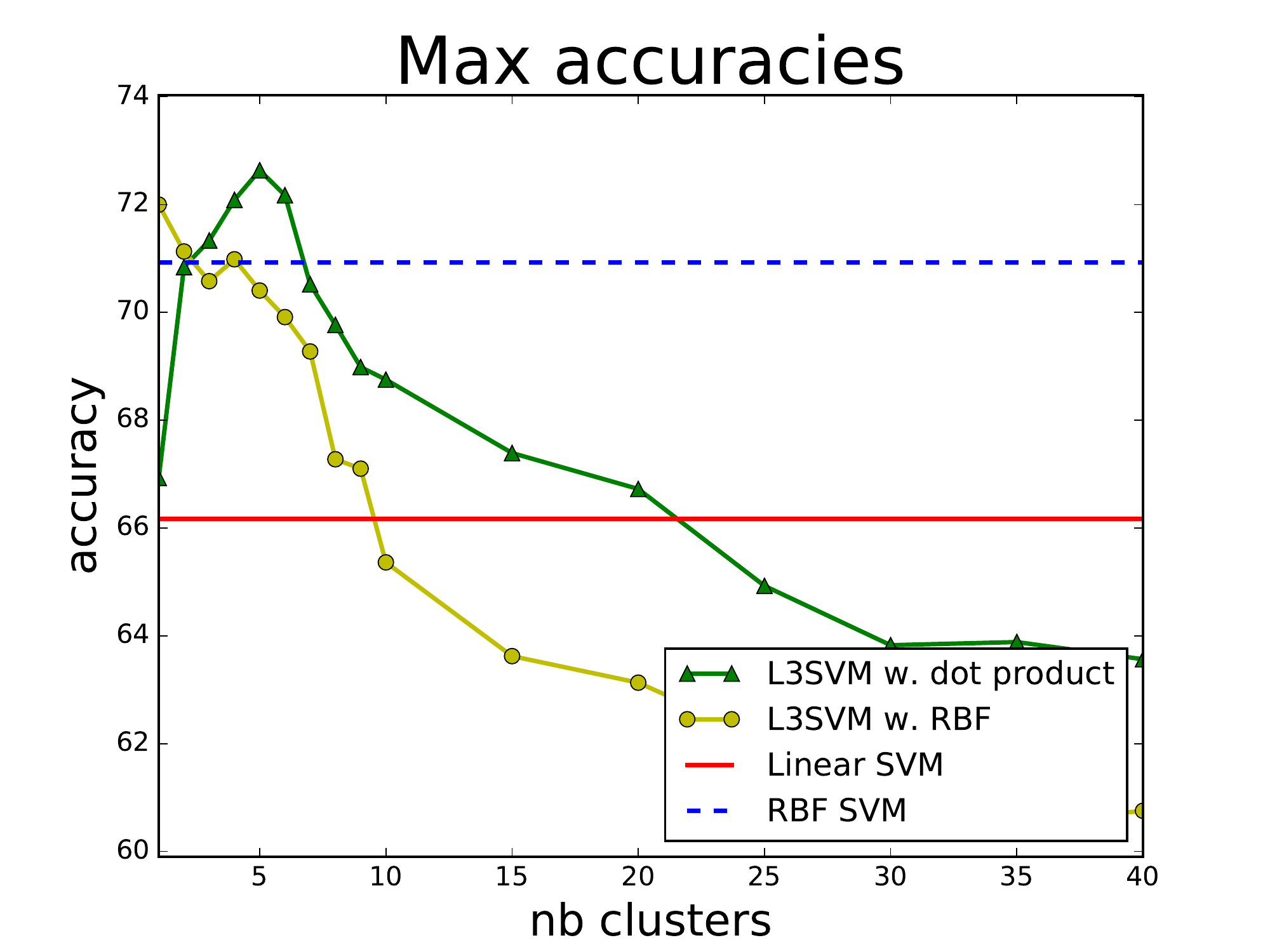}\\
     \caption{Liver: 345 instances, 6 features}
     \label{fig:liver}
   \end{subfigure}

   \begin{subfigure}{\textwidth}
      \centering
        \includegraphics[width=0.30\textwidth]{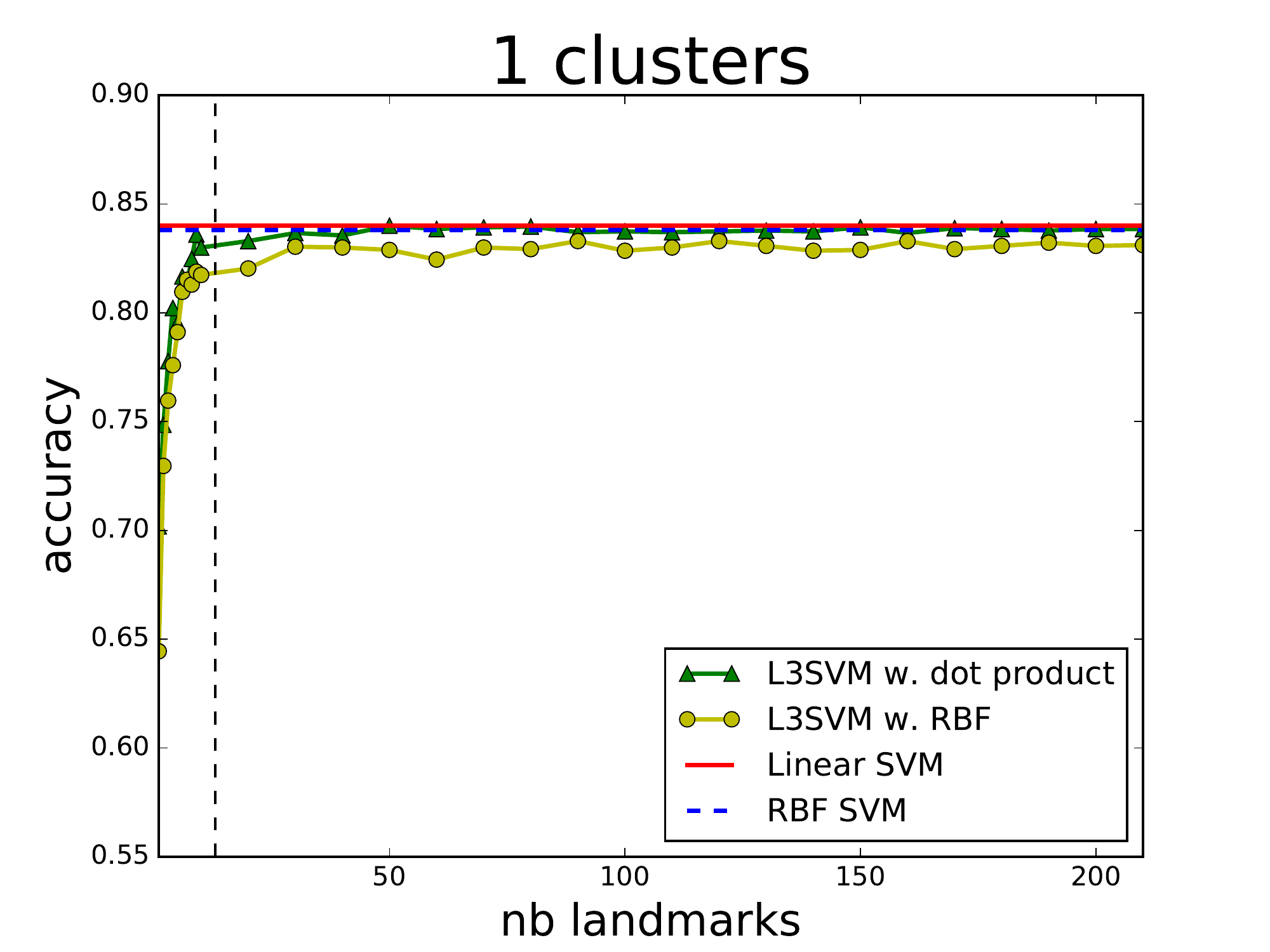}
        \includegraphics[width=0.30\textwidth]{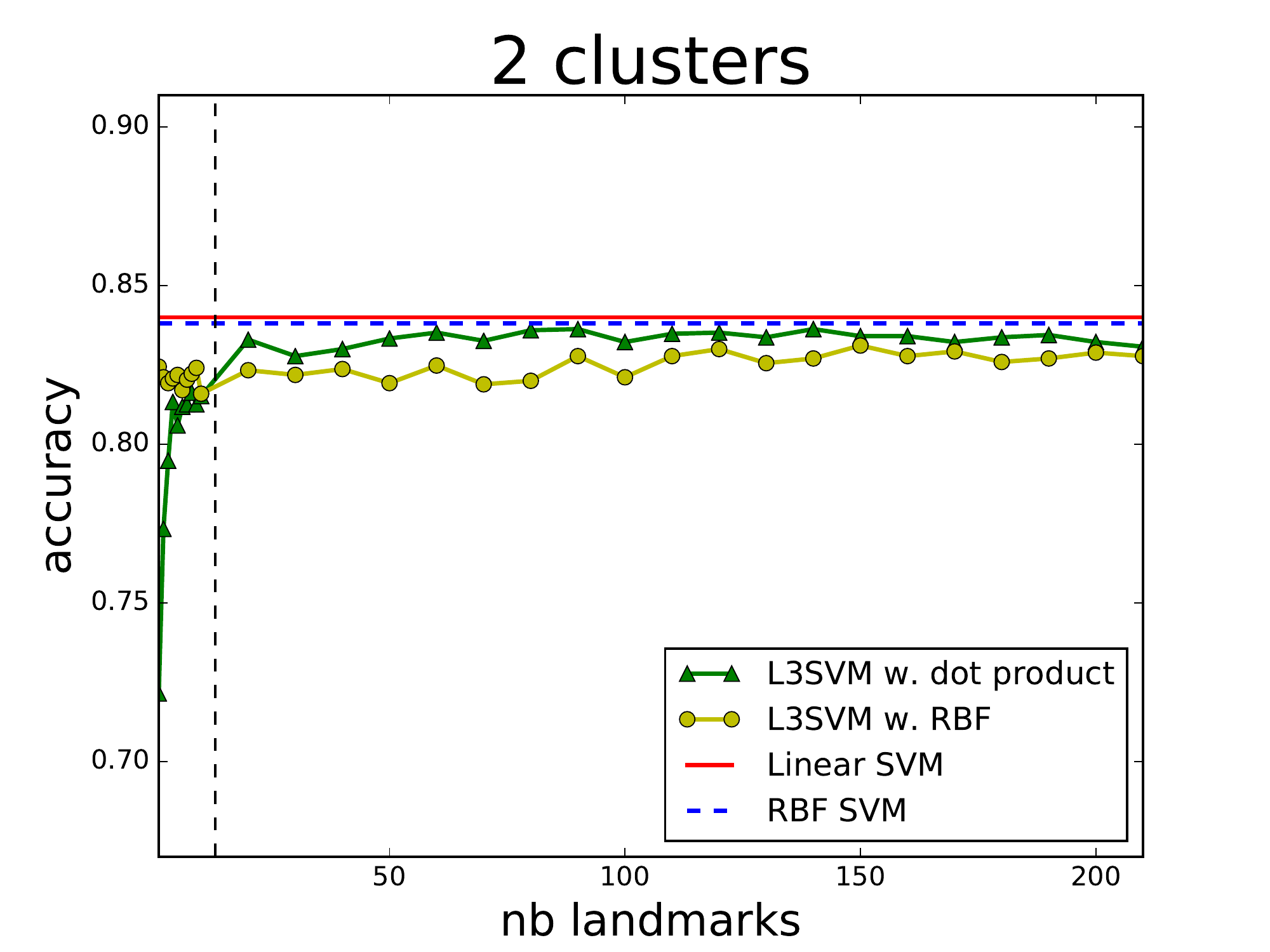}
        \includegraphics[width=0.30\textwidth]{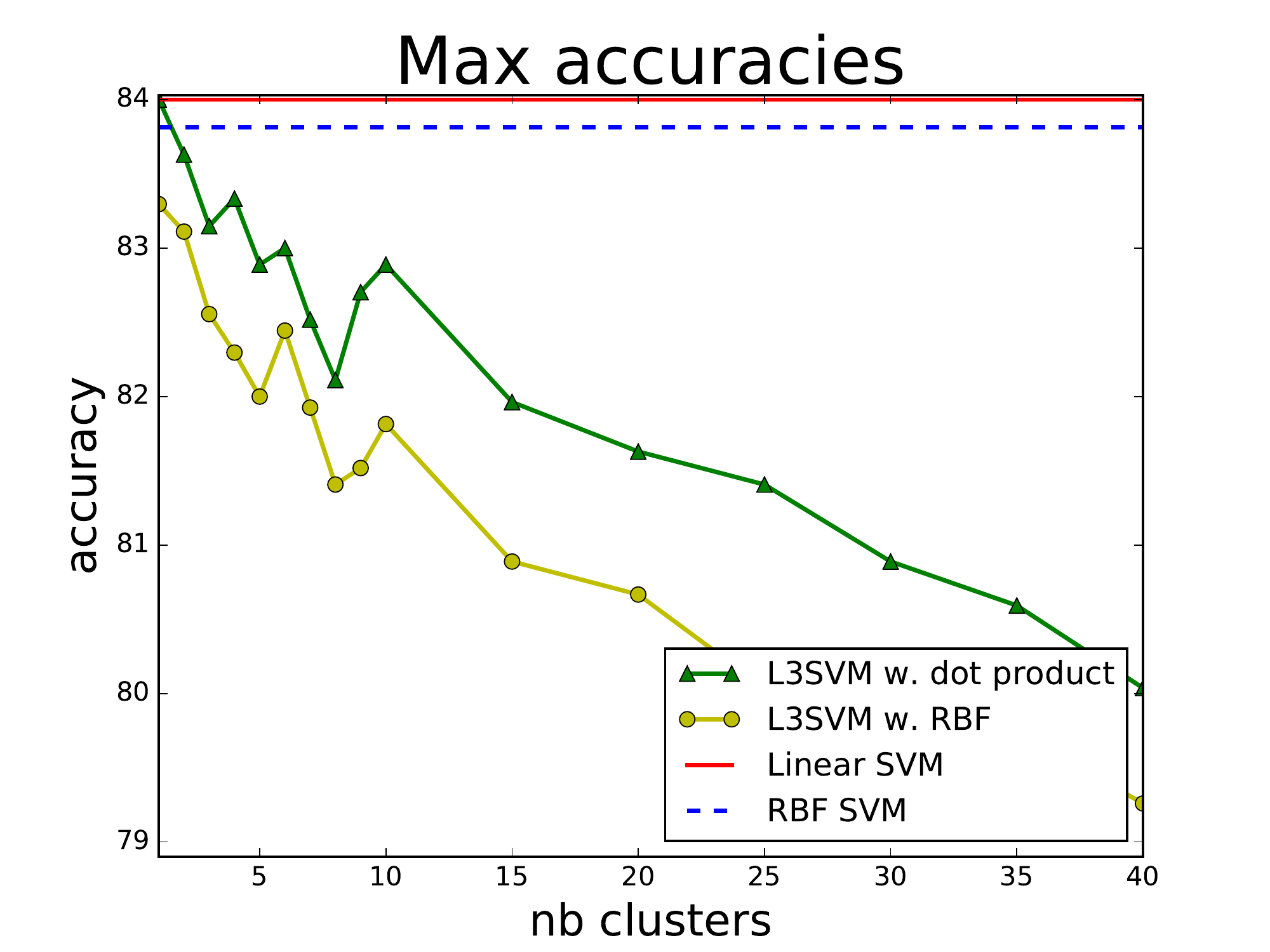}\\
       \caption{Heart-Statlog: 270 instances, 13 features.}
       \label{fig:heart}
    \end{subfigure}
    \begin{subfigure}{\textwidth}
      \centering
        \includegraphics[width=0.30\textwidth]{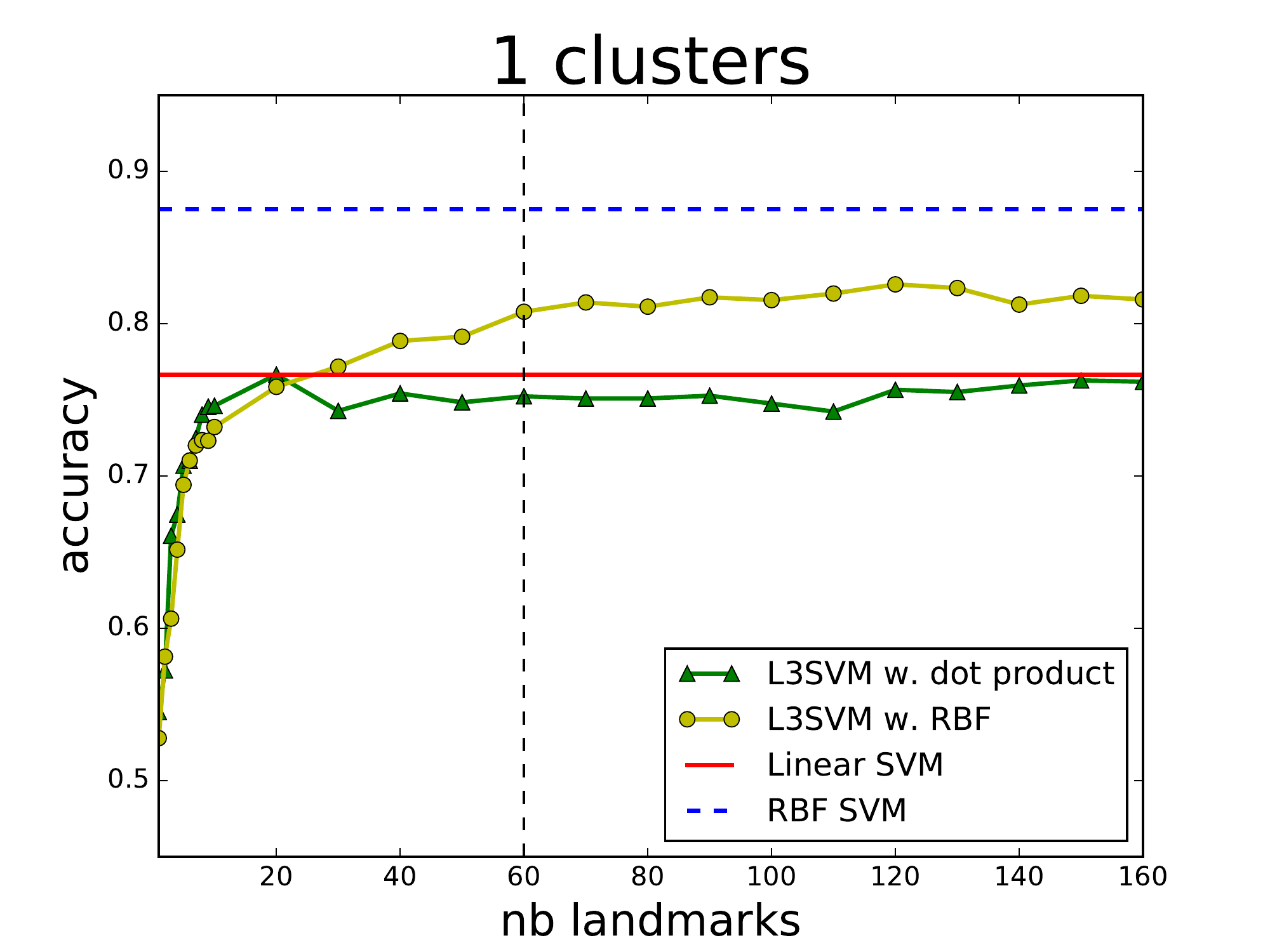}
        \includegraphics[width=0.30\textwidth]{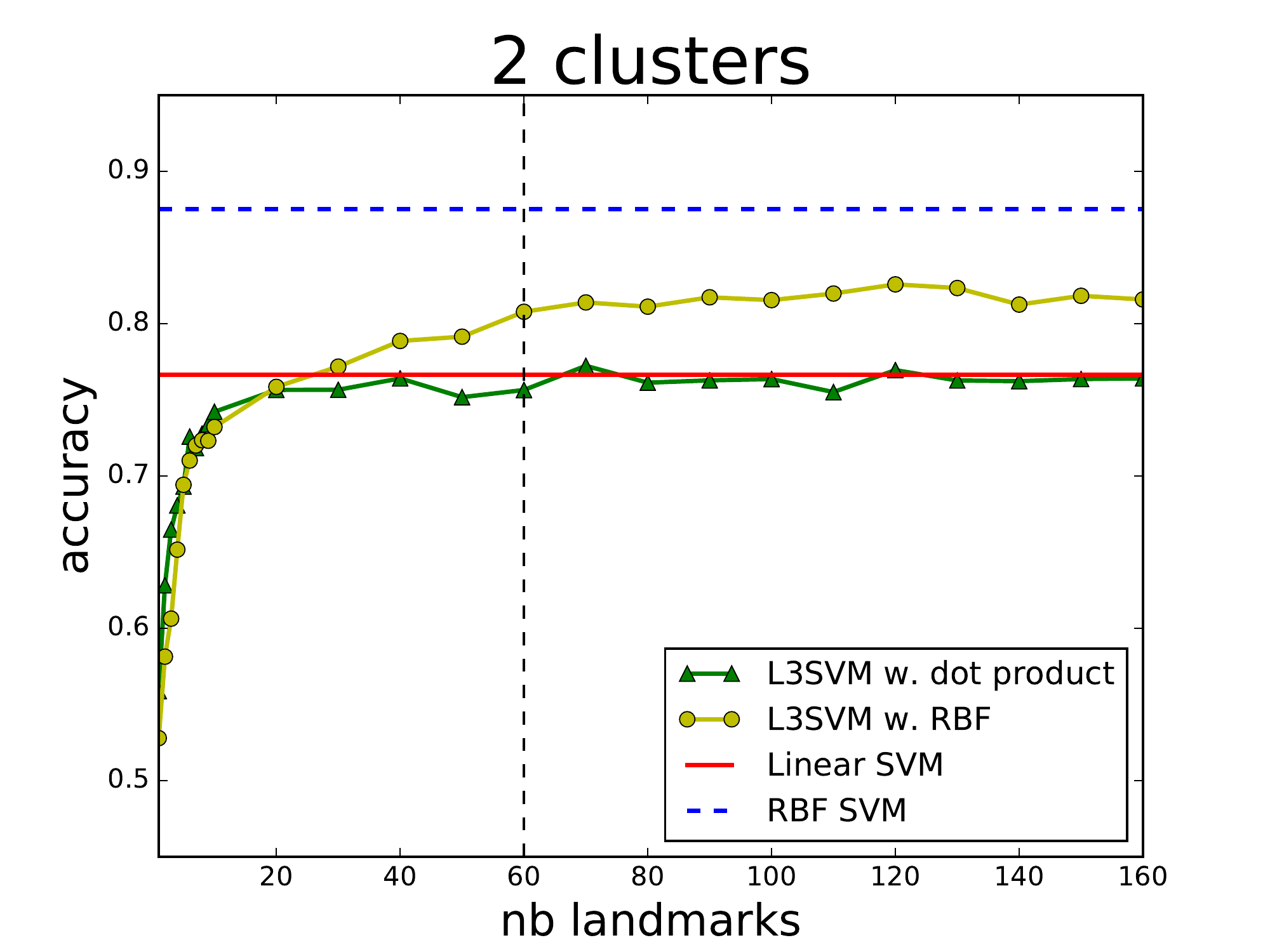}
        \includegraphics[width=0.30\textwidth]{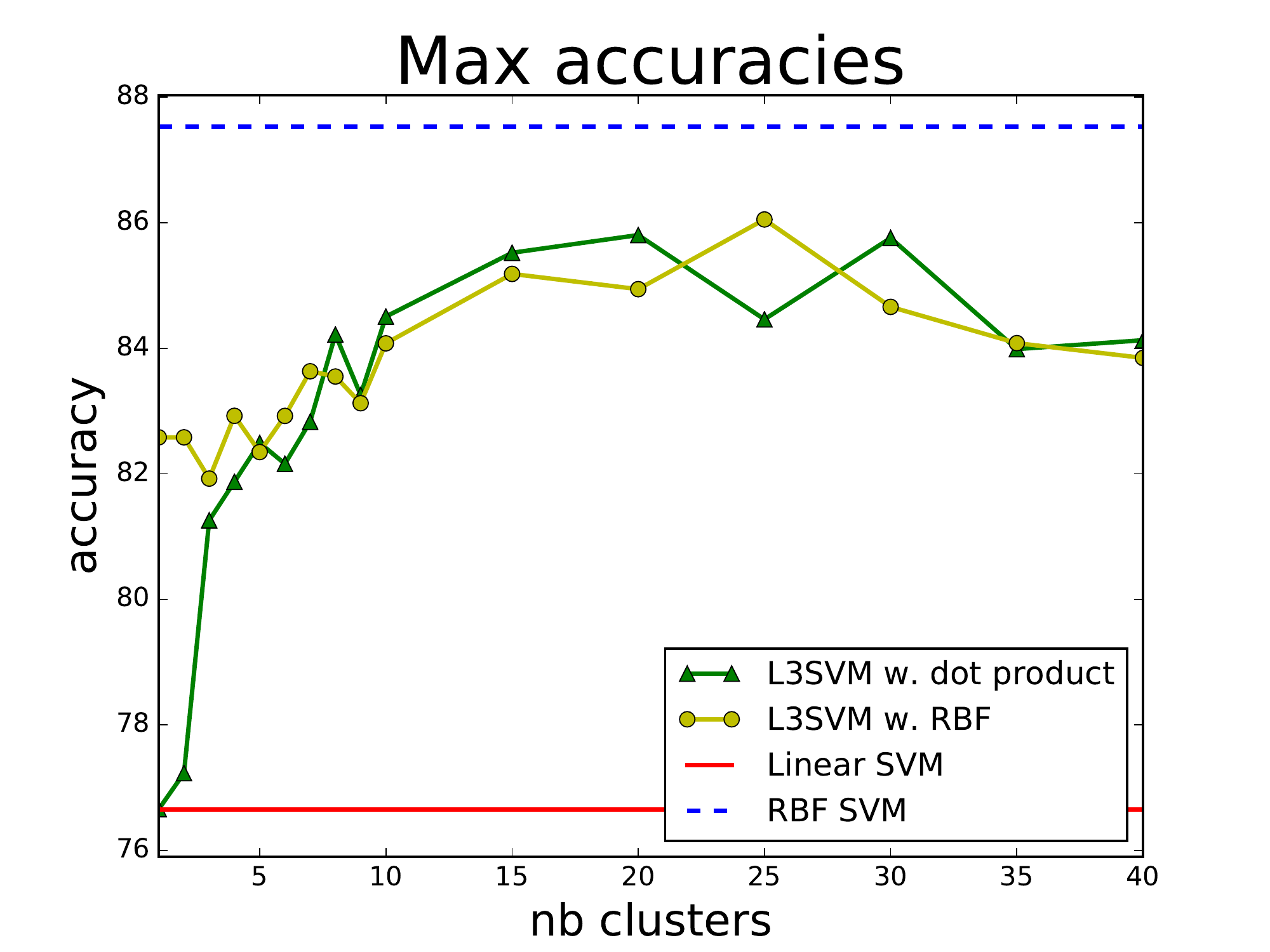}\\
       \caption{Sonar: 351 instances, 60 features.}
       \label{fig:sonar}
    \end{subfigure}
    \begin{subfigure}{\textwidth}
      \centering
        \includegraphics[width=0.30\textwidth]{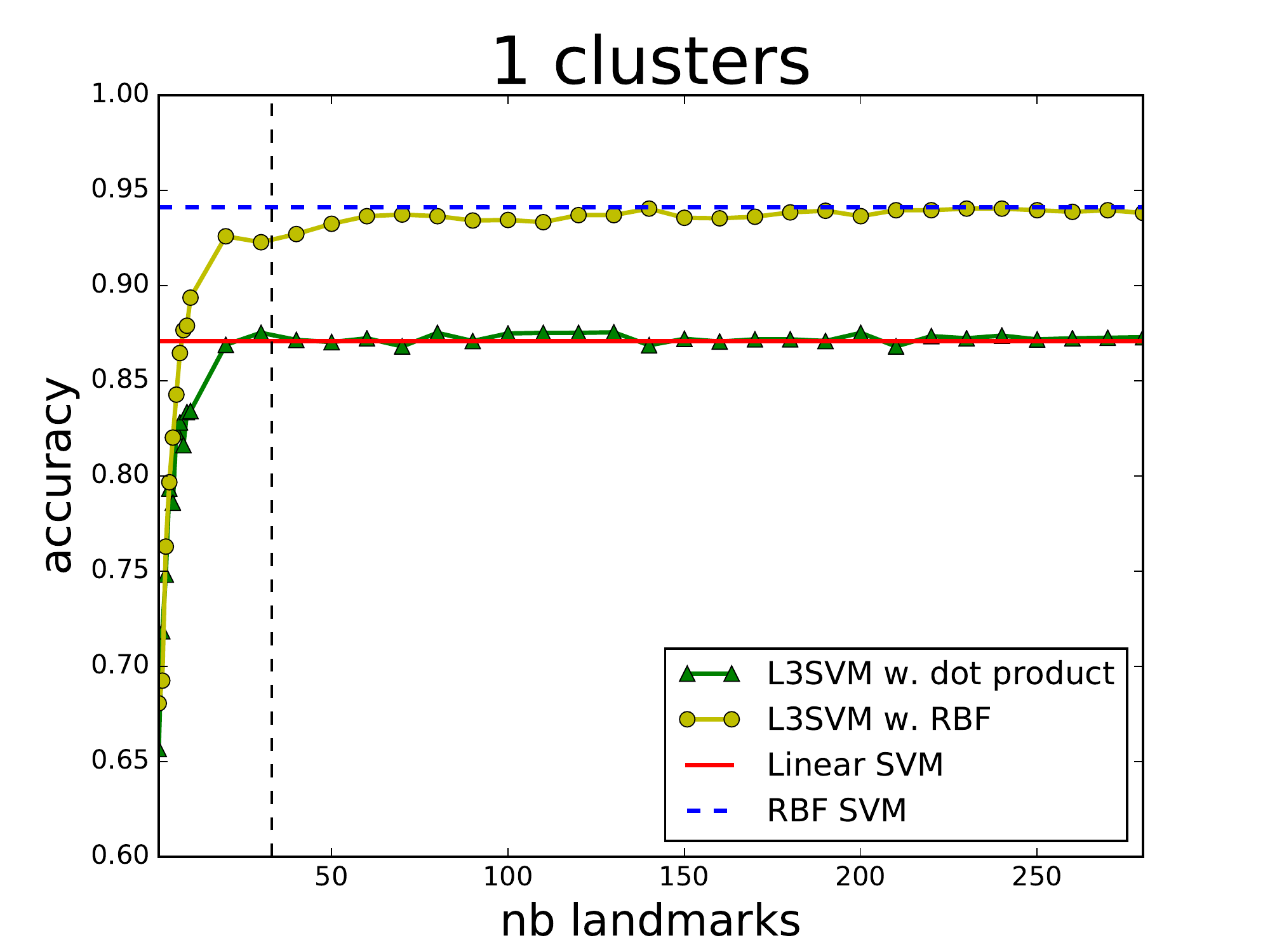}
        \includegraphics[width=0.30\textwidth]{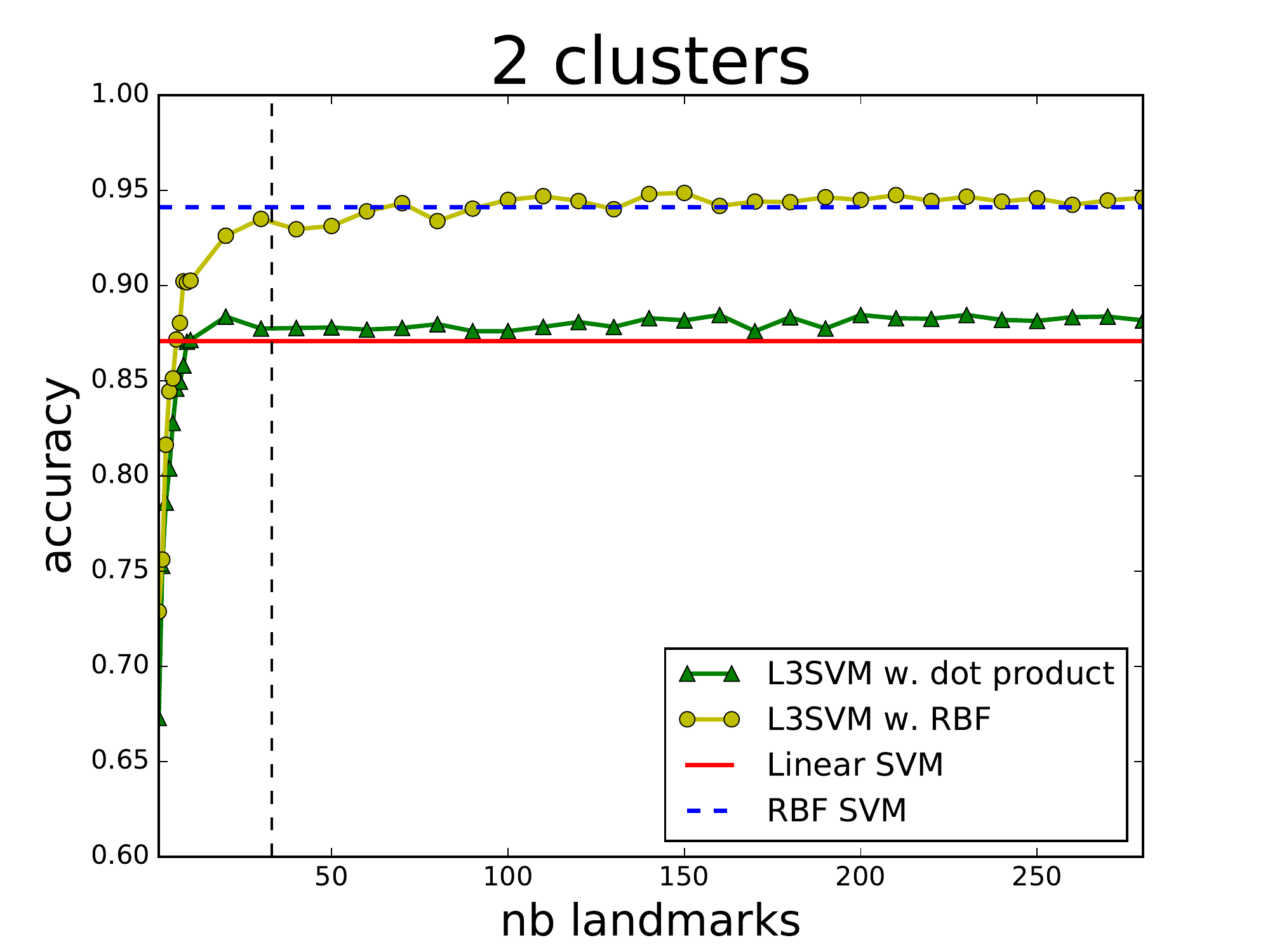}
        \includegraphics[width=0.30\textwidth]{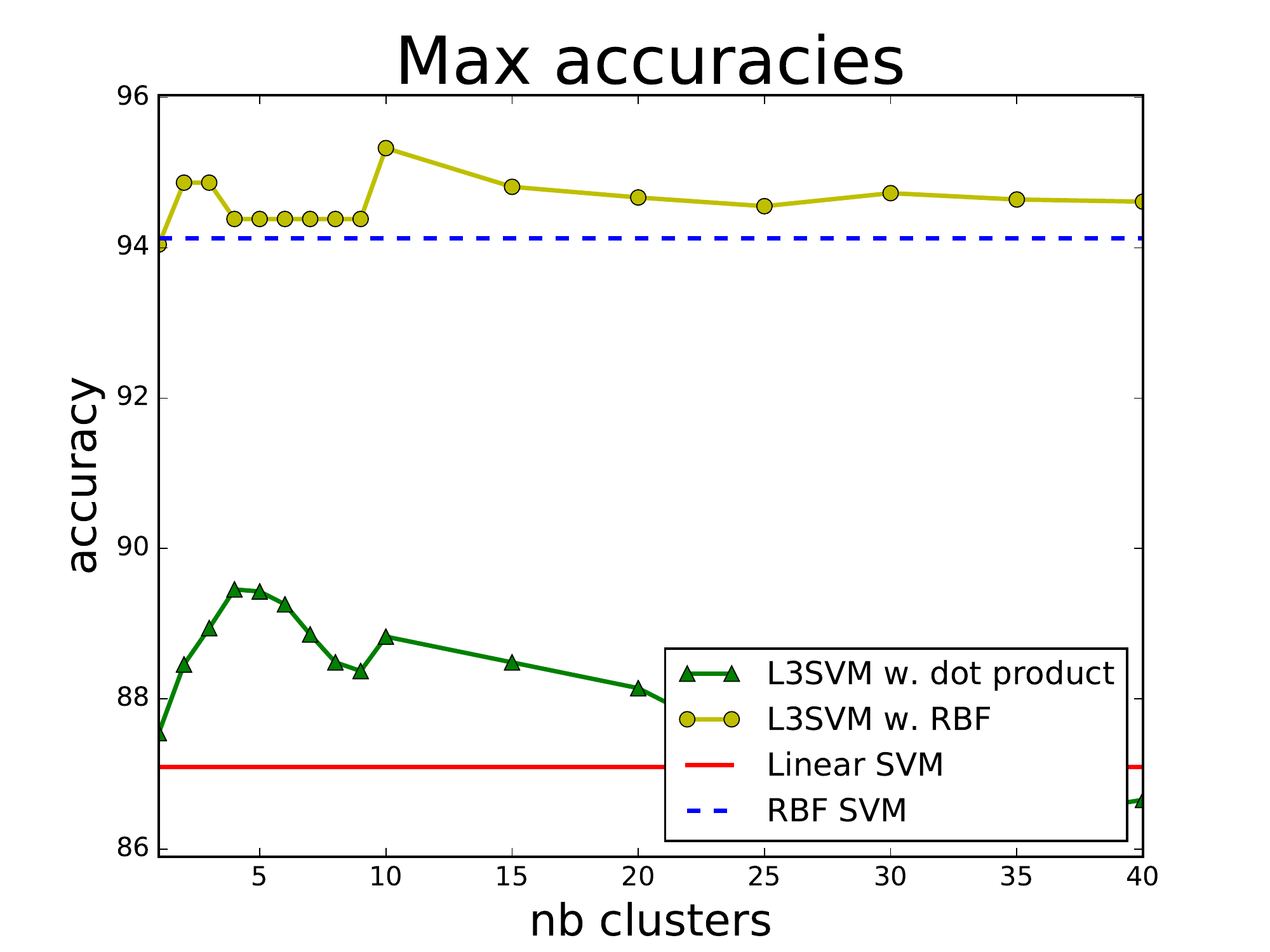}\\
       \caption{Ionosphere: 209 instances, 33 features.}
       \label{fig:ionosphere}
    \end{subfigure}
  \caption{We report the results for 1 cluster (left) and 2 clusters (middle), and for 1 to 40 clusters we report the maximal mean accuracy obtained with all possible values of $L$ with a given number of clusters (right). The black line in the first two pictures marks the dimension of the input space.}
  \label{fig:uci}
\end{figure*}
The aim of the following experiment is to empirically study how the number of landmarks impacts the testing accuracy. To do so, we fix the number of clusters (between 1 and 40) and vary the number of landmarks from 1 to the size of the training sample.

We compare the performances of standard SVMs (linear or kernelized with RBF) with those of \landSVM (using a linear or RBF projection) on three UCI datasets~\cite{Lichman:2013}. In Fig.~\ref{fig:uci} we draw the mean testing accuracies of a 5-fold cross-validation procedure repeated 10 times. For all the methods, at each iteration we tune the hyper-parameters by grid search with the values $\{10^{-3},10^{-2},10^{-1},1,10,100 \}$ with a 5-fold cross-validation procedure and we cluster the instances using k-means.

\paragraph{Liver, Fig.~\ref{fig:liver}} Already with 2 clusters, \landSVM achieves testing accuracies similar to those of a kernelized SVM. Furthermore, our method has the best results for 2 to 6 clusters. On the other hand, it seems that a \landSVM with a RBF projection function is really sensitive to overfitting.

\paragraph{Heart-Statlog, Fig.~\ref{fig:heart}} In this case, learning local models makes the predictions worse than learning a global one. As a matter of fact, from the comparison of an SVM and a Kernel SVM, it seems that the problem is linearly separable and that learning a non-linear classifier does not improve the results. Therefore, increasing the number of local models only makes them overfit.

\paragraph{Sonar, Fig.~\ref{fig:sonar}} Studying this dataset, which has more features than the previous two, it seems that it is possible to select a number of landmarks smaller than the dimension of the input space without deteriorating the results.

\paragraph{Ionosphere, Fig.~\ref{fig:sonar}} With this dataset, our method is not able to capture the non-linearities of the input space by combining local linear models. However, we notice that, already with 1 cluster and a number of landmarks at least equal to the dimension of the space, we obtain similar results by using the RBF kernel and by solving \landSVM with a RBF projection.
\\
 
In conclusion, we  claim that it is not interesting to have a number of landmarks greater than the dimension of the input space and that reducing the number of landmarks is not conceivable on datasets of small number of features. 
Also in this experiment, the performances of \landSVM with a RBF projection function are close or even worse than those of \landSVM with a linear kernel, probably because of overfitting. Therefore, in the following sections, we will restrict our studies only to a \landSVM with linear projection.

\subsection{Dimensionality Reduction}
The aim of the series of experiments is to study the impact of the chosen technique for landmark selection on the performances of our method. We compare  \landSVM with a set of landmarks randomly selected from the training sample to \landSVM with the landmarks as the principal components of the covariance matrix of the training set (performing a PCA) on the MNIST dataset~\cite{lecun-mnisthandwrittendigit-2010}. In Fig.~\ref{fig:pca}, we report the testing accuracies w.r.t. the number of landmarks $L$, as well as the time needed for selecting the landmarks. The number of clusters is fixed to $100$ and the parameter $c$ is tuned by grid search by 5-fold cross-validation. The instances are clustered using k-means. 

\pgfplotsset{
  xmin=10,xmax=784,
  xtick={10,100,200,300,400,500,600,700,784},
  legend style ={ at={(1.1,1)}, 
  anchor=north west, draw=black, 
  fill=white,align=left,font=\tiny},
  height=5cm,
  width=\textwidth
}

\begin{figure*}[t]
\captionsetup{justification=centering}
    \begin{subfigure}{0.45\textwidth}
    \begin{tikzpicture}
        \begin{axis}[xlabel = nb landmarks,height=4cm]
            \addplot table[x=L,y=accuracy,col sep=space,mark=o]{mnist/pca.csv};
            \addplot table[x=L,y=accuracy,col sep=space,mark=x]{mnist/random.csv};
            
        \end{axis}
    \end{tikzpicture}
    \caption{
      Testing Accuracy (\%)
    }
    \end{subfigure}
    % \hspace{-1cm}
    \begin{subfigure}{0.45\textwidth}
    \begin{tikzpicture}
        \begin{axis}[xlabel = nb landmarks,height=4cm]
            \addplot table[x=L,y=time,col sep=space,mark=o]{mnist/pca-time.csv};
            \addlegendentry{PCA};
            \addplot table[x=L,y=time,col sep=space,mark=x]{mnist/random-time.csv};
            \addlegendentry{Random};
        \end{axis}
    \end{tikzpicture}
    \caption{
      Selection Time (s)
    }
    \end{subfigure}
    \caption{
      Comparison of the testing accuracies and selection times (in seconds) for two methods of landmark selections: PCA and random selection. Notice that the difference in accuracy is limited when $L$ is bigger than $100$, while the time complexity is significantly lower using a random selection (around $0.020s$).
    }
    \label{fig:pca}

\end{figure*}
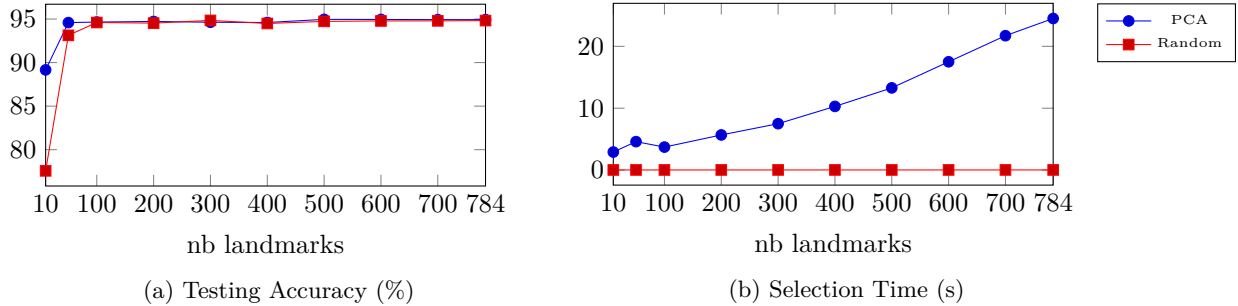

We use the Principal Component Analysis of the scikit-learn package~\cite{scikit-learn}, which implements the randomized SVD presented in~\cite{halko2011finding}. Having denoted $n$ the number of features and $m$ the number of instances, the complexity of this method is at worst $O(mn\log(n) + (m+n)n^2)$, when the rank of the training set is equal to $n$. Compared to a random selection ($O(L)$ as $L \ll m$) a PCA-based selection is more expensive and it achieves better results only when $L < 100$. 

These results suggest that, when $L$ is small, it is interesting to select good landmarks (by means of a PCA for instance) and it can be done in reasonable time. On the other hand, when $L$ is big, there is no need to force the variety and expressiveness of the set of landmarks, and a random selection from the training sample already allows us to have a good projection of the input space with little effort.

\subsection{Comparison with the State of the Art}
In this final series of experiments, we compare \landSVM with state-of-the-art methods on the four datasets presented in Table~\ref{tab:dataset} (with the features rescaled to have a standard deviation of 1).
In all experiments, we fix $L$ to the dimension of the input space, we select the landmarks randomly from the training sample and we cluster using k-Means.

\begin{table}
    \centering
    \caption{Characteristics of Datasets}
    \label{tab:dataset}
    \scalebox{0.7}{
    \setlength\tabcolsep{3pt}
    \begin{tabular}{ | c | c | c | c | c | c |}
      \hline
      & \#training & \#testing & \#features & \#classes & \#models \\ \hline
      \textbf{SVMGUIDE1} & 3089 & 4000 & 4 & 2 & 100\\ \hline
      \textbf{IJCNN1} & 49990 & 91701 & 22 & 2 & 100 \\ \hline
      \textbf{USPS} & 7291 & 2007 & 256 & 10 & 80 \\ \hline
      \textbf{MNIST} & 60000 & 10000 & 784 & 10 & 90 \\ \hline
    \end{tabular}}
\end{table}

\begin{table*}[h!]
    \centering
    \caption{Testing Accuracies (\%)}
    \label{tab:accuracy}
    \scalebox{0.9}{
    \begin{tabular}{ | c | c | c | c | c |}
      \hline
      & \textbf{SVMGUIDE1} & \textbf{IJCNN1} & \textbf{USPS} & \textbf{MNIST} \\ \hline
      \textbf{RBF-SVM} & 96.53 & 97.08 & 94.07 & 96.62\\ \hline
      \textbf{Linear-SVM} & 95.38 & 89.68 & 91.72 & 91.8\\ \hline
      \textbf{CSVM} & 95.05 & 96.35 & N/A & N/A\\ \hline
      \textbf{LLSVM} & 94.08 & 92.93 & 75.69 & 88.65\\ \hline
      \textbf{ML3} & 96.68 & 97.73 & 93.22 & 97.04\\ \hline
      \textbf{L$^3$-SVMs} & 95.73 & 95.74 & 92.12 & 95.05\\
      \hline
    \end{tabular}}
\end{table*}

\begin{table*}[h!]
    \centering
    \caption{Training and Testing times (in seconds).}
    \label{tab:timing}
    \scalebox{0.9}{
    \begin{tabular}{ | c | c | c | c | c |}
      \hline
      & \textbf{SVMGUIDE1} & \textbf{IJCNN1} & \textbf{USPS} & \textbf{MNIST} \\ \hline
      \textbf{RBF-SVM} & \twocol{0.39}{0.11} & \twocol{104.02}{23.32} & \twocol{44.34}{3.49} & \twocol{2699.42}{136.82}\\ \hline
      \textbf{Linear-SVM} & \twocol{0.04}{0.06} & \twocol{0.74}{2.31} & \twocol{1.44}{0.36} & \twocol{24.29}{3.93}\\ \hline
      \textbf{CSVM} & \twocol{1.54}{0.06} & \twocol{2.31}{3.18} & N/A & N/A\\ \hline
      \textbf{LLSVM} & \twocol{0.23}{0.05} & \twocol{6.21}{0.93} & \twocol{121.83}{2.26} & \twocol{1393.58}{9.27}\\ \hline
      \textbf{ML3} & \twocol{1.47}{0.04} & \twocol{17.71}{1.01} & \twocol{41.29}{0.25} & \twocol{1312.48}{6.22} \\ 
      \hline
      \textbf{L$^3$-SVMs} & \twocol{0.22}{0.13} & \twocol{14.08}{5.35} & \twocol{34.91}{0.98} & \twocol{276.15}{9.29} \\ \hline
    \end{tabular}}
\end{table*}

Table~\ref{tab:accuracy} (resp.~\ref{tab:timing}) report the accuracy (resp. running times) of  \landSVM method and standard SVMs using either a linear or RBF kernel (using Liblinear or Libsvm~\cite{chang2011libsvm}), Clustered SVM (CSVM)~\cite{gu2013clustered}, Locally Linear SVM (LLSVM)~\cite{ladicky2011locally} and ML3 SVM~\cite{fornoni2013multiclass}\footnote{The results of CSVM for the multi-class datasets are missing because it is implemented only for binary classification.}. The number of local models is fixed and, if not differently specified in the respective papers (such as 8 nearest neighbors for LLSVM and $p=1.5$ for ML3), the hyper-parameters are tuned by 5-fold cross-validation.

While the testing time is sometimes higher than the other methods, it can be reduced by limiting the number of landmarks for the datasets with a lot of features, as in the previous experiments we showed that it doesn't affect the results. 
Overall, our method compares favorably in terms of training time, especially for high-dimensional input spaces, and has good accuracy across all datasets.

\section{Conclusions and Perspectives}
\label{sec:conclpersp}

We introduce a new local learning algorithm named \landSVM.
It relies on a partitioning of the input space and on a projection of all points onto a set of landmarks.
Using the uniform stability framework, we show that \landSVM has theoretically generalization guarantees.
The empirical evaluation highlights that \landSVM is fast while being competitive with the state of the art.

While we introduced \landSVM with its ``default'' choices, the algorithm offers a lot of exciting perspectives.
First, we can refine many of the elements of \landSVM:
the partitioning using k-means can be replaced by other existing hard or soft clustering algorithms;
the random landmark selection procedure could be improved, for example using methods like DSELECT~\cite{kar2011similarity} and Stochastic Neighbor Compression~\cite{kusner2014stochastic}, or using density estimation~\cite{liu2016stein};
at a greater computational cost, a non-linear kernel can be used to have two levels of non-linearities (see Section~\ref{sec:where:cankerneltrick}).
% TODO: maybe, mention again the faster optimization method from Sec. 2.2
Even if the common landmarks act as a regularization of the local models, an overfitting is observed when the number of clusters becomes high.
The model could naturally accept explicit spatial regularization terms to increase the spatial smoothness of the models across clusters.
% TODO: autocite CVPR?
The speed and linearity of \landSVM also open the door to an auto-context approach (stacking): \landSVM would be reapplied on the data after projecting it on the previous level's support vectors.
Beyond stacking, we plan to explore a deep version of the algorithm, where the intermediate layers of projection are learned in a joint optimization problem.

% for us: full study and theory of the impact of the preprocessing
% for remi: link to deep half random (residual) models

\appendix
\section{Appendix}

\subsection{Hilbert Space $\mathcal{H}$}
\label{an:hilbert}

\begin{defn}{(\bf{Hilbert Space})}
    A real vector space $\mathcal{V}$ over $\mathbb{R}$ is a Hilbert Space if:
    \begin{enumerate}
        \item $\mathcal{V}$ is a real inner product space;
        \item $\mathcal{V}$ is a complete metric space with respect to the distance function induced by its inner product.
    \end{enumerate}
\end{defn}

\begin{thm}
    The space $\mathcal{H}$ resulting by a transformation $\mul{x} = [\mu(x,l_1),...,\mu(x,l_L)]$, with $\mu : \mathcal{X}^2 \to \mathbb{R}$ of an Hilbert space $\mathcal{X}$ is also an Hilbert Space if $\mathcal{L} \neq \bm{0}$.

\end{thm}

\begin{proof}

~\\If $\mathcal{L} \neq \bm{0}$, $<\mul{},\mul{}> = \mul{}\mul{}^T$ is an inner product, as:

\begin{enumerate}
    
    \item $<\mul{},\mul{}>$ is linear: $ \forall a,b \in \mathbb{R}$ and $ \forall x_1,x_2,x_3 \in \mathcal{X}$

    \small{
    \begin{align*}
        <a & \mul{x_1}+b\mul{x_2},\mul{x_3}> \\
        &= \big( a\mul{x_1} + b\mul{x_2} \big)\mul{x_3}^T \\
        &= a\mul{x_1}\mul{x_3}^T + b\mul{x_2}\mul{x_3}^T \\
        &= a <\mul{x_1},\mul{x_3}> + b <\mul{x_2}\mul{x_3}>;
    \end{align*}
    }

    \item $<\mul{},\mul{}>$ is symmetric: $ \forall x_1, x_2 \in \mathcal{X}$

    $$ <\mul{x_1},\mul{x_2}> = <\mul{x_2},\mul{x_1}>;$$

    \item $<\mul{},\mul{}>$ is always non-negative and null only for $\bm{x}=\bm{0}$: $\forall x \in \mathcal{X}$

    $$<\mul{x},\mul{x}> = \sump \mu(x,p)^2 \geq 0$$ 

    and $<\mul{x},\mul{x}> = 0$ iff $\bm{x}=\bm{0}$ as $\mathcal{L} \neq \bm{0}$.

\end{enumerate}

\end{proof}

In particular, the space generated by $\mu(x_1,x_2) = x_1^Tx_2$ or $\mu(x_1,x_2) = \exp(-\frac{\normtwo{x_1-x_2}^2}{\sigma})$ is an Hilbert Space. 

\subsection{Lagrangian Dual Problem}
\label{an:dual}

The \landSVM optimization problem takes the following form:
$$ \argmin_{\theta,b,\xi} \frac{1}{2} \normf{\theta}^2 + \frac{c}{m} \sumi \xi_i$$
$$s.t. \: y_i \left(\theta_{k_{i.}} \mul{x_i}^T + b \right) \geq 1- \xi_i \:\: \forall i=1..m$$
$$\xi_i \geq 0 \:\: \forall i=1..m$$ \label{eq:primal}

with $\mul{.} = [\mu(.,l_1),...,\mu(.,l_L)]$ the projection from the input space $\mathcal{X}$ to the landmark space $\mathcal{H}$.

The Lagrangian dual problem of the previous formulation is obtained by maximizing the corresponding Lagrangian \wrt its Lagrangian multipliers. The derived problem is a Quadratic Programming problem that can be solved by common optimization techniques and that allows one to make use of the kernel trick. The Lagrangian takes the following form:

$$ \mathcal{L}(\theta,b,\xi,\alpha,r) = \frac{1}{2} \normf{\theta}^2 + \frac{c}{m}\sumi \xi_i-\sumi r_i \xi_i -\sumi \alpha_i \left(y_i \big(\theta_{k_{i.}} \mul{x_i}^T + b \big) + \xi_i-1\right)$$
where $\alpha \in \mathbb{R}^{m}$ and $r \in \mathbb{R}^{m}$ are the positive Lagrangian multipliers.
Let's consider the fact that:

$$ \max_{\alpha,r} \min_{\theta,b,\xi} \mathcal{L}(\theta,b,\xi,\alpha,r) \leq \min_{\theta,b,\xi} \max_{\alpha,r} \mathcal{L}(\theta,b,\xi,\alpha,r) $$
where the left term corresponds to the optimal value of the dual problem and the right one to the primal's one. The dual and the primal problems have the same value at optimality if the Karush-Kuhn-Tucker (KKT) conditions are not violated (see~\cite{boyd2004convex}).

By setting the gradient of $\mathcal{L}$ \wrt $\theta, b$ and $\xi$ to 0, we find the saddle point corresponding to the function minimum:
$$ \nabla_{\theta_{kp}}\mathcal{L}(\theta,b,\xi,\alpha,r) = \theta_{kp} - \sumik{i} \alpha_i y_i \mu(x_i,l_p)$$

$$\nabla_{b}\mathcal{L}(\theta,b,\xi,\alpha,r) = - \sumi \alpha_i y_i $$

$$\nabla_{\xi_i}\mathcal{L}(\theta,b,\xi,\alpha,r) = \frac{c}{m} - \alpha_i - r_i$$

which give
\begin{equation} \label{eq:theta}
\theta_{kp} = \sumik{i} \alpha_i y_i \mu(x_i,l_p)
\end{equation}

\begin{equation} \label{eq:b}
\sumi \alpha_i y_i = 0
\end{equation}

\begin{equation} \label{eq:r}
\alpha_i = \frac{c}{m} - r_i
\end{equation}

We can now write the QP dual problem by replacing $\theta$ by its expression~\eqref{eq:theta} and simplifying following~\eqref{eq:b} and~\eqref{eq:r}:
$$ \max_{\alpha} \:\: -\frac{1}{2}\sumik{i}\sumik{j} \alpha_i \alpha_j y_i y_j \mul{x_i}\mul{x_j}^T + \sumi \alpha_i$$

$$ s.t. \:\: 0 \leq \alpha_i \leq \frac{c}{m} \:\: \forall i=1..m$$
$$ \sumi \alpha_i y_i = 0 \:\: \forall i=1..m$$

which is concave \wrt $\alpha$.

We need the following two additional constraints in order to respect the KKT conditions which guarantee that the optimal value found by solving the dual problem corresponds to the optimal value of the primal:
$$\alpha_i \left( y_i \left(\theta_{k_{i.}} \mul{x_i}^T + b \right) -1 + \xi_i \right) = 0 \:\: \forall i=1..m$$
$$r_i \xi_i = 0 \:\: \forall i=1..m$$

Once the Lagrangian dual problem solved, the characteristic vector $\theta$ and offset $b$ of the optimal margin hyperplane can be retrieved by means of the support vectors, i.e. the instances whose corresponding $\alpha_i$ are strictly greater than $0$:
$$\theta_{kp} = \sumik{a} \alpha_a y_a \mu(x_a,l_p)$$
$$b = y_a - \theta_{k_{a.}} \mul{x_a}$$
and the new instances can be classified :
$$y(x) = sign \left( \theta_{k_{i.}} \mul{x_i}^T + b \right).$$

\subsection{Graphical representation of variable dependencies}
\label{an:graphicalmodels}

Figures~\ref{fig:gm-svmpercluster} through~\ref{fig:gm-oursvm} graphically illustrates the variables involved in the different optimization problems that are solved by the local SVM approaches and \landSVM.
In these graphs, a node represents a variable (or a set of) and a link show a direct dependency between the variables, i.e., one variable is directly involved in the computation or the estimation of the other. 

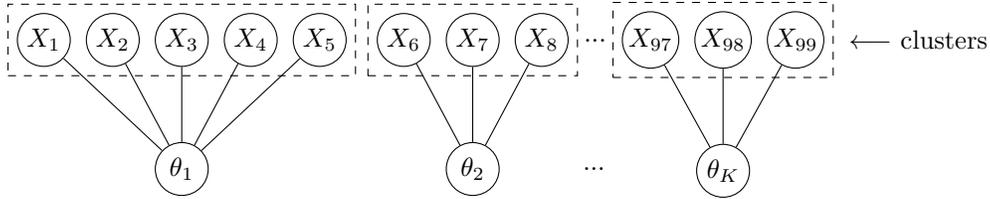
\begin{figure}[h!]
  \centering
  \begin{tikzpicture}[x=0.2cm]

  \FPeval{\spahalf}{clip(0.9)}
  \FPeval{\spamore}{clip(0.2)}

  \node[latent] (x-1)  {$X_1$} ;
  \foreach \i in {2,...,5}{
    \FPeval{\prev}{clip(\i - 1)}
    \node[latent, right=of x-\prev] (x-\i)  {$X_{\i}$} ;
  }
  \node[latent, right=of x-5, xshift=\spamore cm] (x-6)  {$X_6$} ;
  \foreach \i in {7,...,8}{
    \FPeval{\prev}{clip(\i - 1)}
    \node[latent, right=of x-\prev] (x-\i)  {$X_{\i}$} ;
  }
  \node[const, right=of x-8] (x-ellipsis)  {...} ;
  \node[latent, right=of x-ellipsis] (x-97)  {$X_{97}$} ;
  \node[latent, right=of x-97]       (x-98)  {$X_{98}$} ;
  \node[latent, right=of x-98]       (x-99)  {$X_{99}$} ;

  \gate {} {(x-1)(x-2)(x-3)(x-4)(x-5)} {}
  \gate {} {(x-6)(x-7)(x-8)} {}
  \gate {lastcluster} {(x-97)(x-98)(x-99)} {}
  \node[const, right=of lastcluster]    (cl-comment)  {$\longleftarrow$ clusters} ;

  \node[latent, below=of x-3]    (t-1)  {$\theta_{1}$} ;
  \foreach \i in {1,...,5}{
    \edge[-] {x-\i} {t-1}
  }
  \node[latent, below=of x-7]    (t-2)  {$\theta_{2}$} ;
  \foreach \i in {6,...,8}{
    \edge[-] {x-\i} {t-2}
  }
  \node[const, right=of t-2, xshift=\spahalf cm] (t-ellipsis)  {...} ;
  \node[latent, below=of x-98]    (t-k)  {$\theta_{K}$} ;
  \foreach \i in {97,...,99}{
    \edge[-] {x-\i} {t-k}
  }
  
\end{tikzpicture}
  \caption{Variable dependencies when learning one SVM per cluster (baseline used in Clustered SVM~\cite{gu2013clustered}).}
  \label{fig:gm-svmpercluster}
\end{figure}

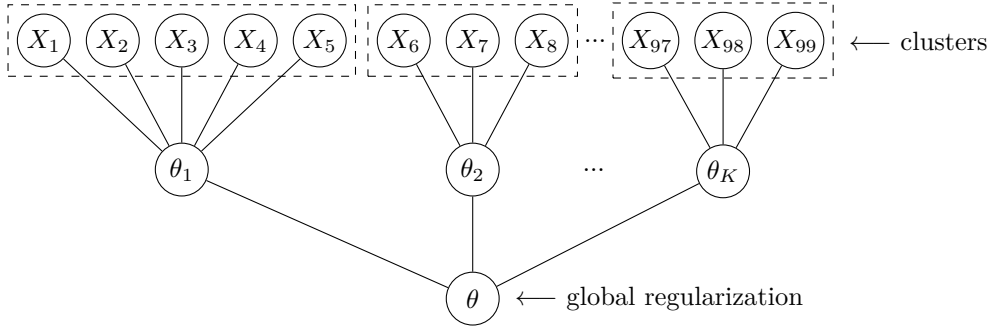
\begin{figure}[h!]
  \centering
  \begin{tikzpicture}[x=0.2cm]

  \FPeval{\spahalf}{clip(0.9)}
  \FPeval{\spamore}{clip(0.2)}

  \node[latent] (x-1)  {$X_1$} ;
  \foreach \i in {2,...,5}{
    \FPeval{\prev}{clip(\i - 1)}
    \node[latent, right=of x-\prev] (x-\i)  {$X_{\i}$} ;
  }
  \node[latent, right=of x-5, xshift=\spamore cm] (x-6)  {$X_6$} ;
  \foreach \i in {7,...,8}{
    \FPeval{\prev}{clip(\i - 1)}
    \node[latent, right=of x-\prev] (x-\i)  {$X_{\i}$} ;
  }
  \node[const, right=of x-8] (x-ellipsis)  {...} ;
  \node[latent, right=of x-ellipsis] (x-97)  {$X_{97}$} ;
  \node[latent, right=of x-97]       (x-98)  {$X_{98}$} ;
  \node[latent, right=of x-98]       (x-99)  {$X_{99}$} ;

  \gate {} {(x-1)(x-2)(x-3)(x-4)(x-5)} {}
  \gate {} {(x-6)(x-7)(x-8)} {}
  \gate {lastcluster} {(x-97)(x-98)(x-99)} {}
  \node[const, right=of lastcluster]    (cl-comment)  {$\longleftarrow$ clusters} ;

  \node[latent, below=of x-3]    (t-1)  {$\theta_{1}$} ;
  \foreach \i in {1,...,5}{
    \edge[-] {x-\i} {t-1}
  }
  \node[latent, below=of x-7]    (t-2)  {$\theta_{2}$} ;
  \foreach \i in {6,...,8}{
    \edge[-] {x-\i} {t-2}
  }
  \node[const, right=of t-2, xshift=\spahalf cm] (t-ellipsis)  {...} ;
  \node[latent, below=of x-98]    (t-k)  {$\theta_{K}$} ;
  \foreach \i in {97,...,99}{
    \edge[-] {x-\i} {t-k}
  }

  % addition to svmpercluster
  \node[latent, below=of t-2]    (t-global)  {$\theta$} ;
  \edge[-] {t-1,t-2,t-k} {t-global}
  \node[const, right=of t-global]    (t-global-comment)  {$\longleftarrow$ global regularization} ;
  
\end{tikzpicture}
  \caption{Variable dependencies for Clustered SVM~\cite{gu2013clustered}, where a common global regularization is used.}
  \label{fig:gm-clusteredsvm}
\end{figure}

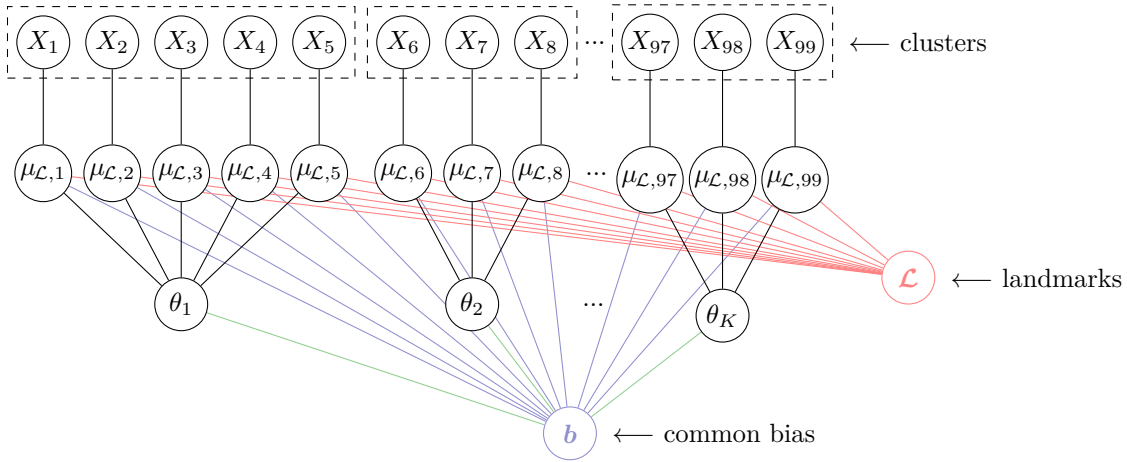
\begin{figure}[h!]
  \centering
  \newcommand{\mymul}[1]{{\mu}_{\mkern-0.66\thinmuskip\mathcal{L}, #1} }

\pgfdeclarelayer{bg}
\pgfsetlayers{bg,main}

\begin{tikzpicture}[x=0.2cm]

  \FPeval{\spahalf}{clip(0.9)}
  \FPeval{\spamore}{clip(0.2)}

  \node[latent] (x-1)  {$X_1$} ;
  \node[latent, below=of x-1] (mu-1)  {$\mymul{1}$} ;
  \foreach \i in {2,...,5}{
    \FPeval{\prev}{clip(\i - 1)}
    \node[latent, right=of x-\prev] (x-\i)  {$X_{\i}$} ;
    \node[latent, below=of x-\i] (mu-\i)  {$\mymul{\i}$} ;
  }
  \node[latent, right=of x-5, xshift=\spamore cm] (x-6)  {$X_6$} ;
  \node[latent, below=of x-6] (mu-6)  {$\mymul{6}$} ;
  \foreach \i in {7,...,8}{
    \FPeval{\prev}{clip(\i - 1)}
    \node[latent, right=of x-\prev] (x-\i)  {$X_{\i}$} ;
    \node[latent, below=of x-\i] (mu-\i)  {$\mymul{\i}$} ;
  }
  \node[const, right=of x-8] (x-ellipsis)  {...} ;
  \node[latent, right=of x-ellipsis] (x-97)  {$X_{97}$} ;
  \node[latent, right=of x-97]       (x-98)  {$X_{98}$} ;
  \node[latent, right=of x-98]       (x-99)  {$X_{99}$} ;
  \node[const, right=of mu-8] (mu-ellipsis)  {...} ;
  \node[latent, below=of x-97] (mu-97)  {$\mymul{97}$} ;
  \node[latent, below=of x-98] (mu-98)  {$\mymul{98}$} ;
  \node[latent, below=of x-99] (mu-99)  {$\mymul{99}$} ;
  \foreach \i in {1,...,8,97,98,99}{
    \edge[-] {x-\i} {mu-\i}
  }

  \gate {} {(x-1)(x-2)(x-3)(x-4)(x-5)} {}
  \gate {} {(x-6)(x-7)(x-8)} {}
  \gate {lastcluster} {(x-97)(x-98)(x-99)} {}
  \node[const, right=of lastcluster]    (cl-comment)  {$\longleftarrow$ clusters} ;

  % thetas
  \node[latent, below=of mu-3]    (t-1)  {$\theta_{1}$} ;
  \node[latent, below=of mu-7]    (t-2)  {$\theta_{2}$} ;
  \node[const, right=of t-2, xshift=\spahalf cm]   (t-ellipsis)  {...} ;
  \node[latent, below=of mu-98]    (t-k)  {$\theta_{K}$} ;
  \foreach \i in {1,...,5}{
    \edge[-] {mu-\i} {t-1}
  }
  \foreach \i in {6,...,8}{
    \edge[-] {mu-\i} {t-2}
  }
  \foreach \i in {97,...,99}{
    \edge[-] {mu-\i} {t-k}
  }

  % below
  \begin{pgfonlayer}{bg}
    \newcommand{\colorforb}{lightgray!75!blue}
    \newcommand{\colorforbb}{lightgray!75!green}
    \newcommand{\colorforL}{white!50!red}
    % b
    \node[latent, text=\colorforb, draw=\colorforb, below=of t-2, xshift=1.3cm]    (b)  {$\bm{b}$} ;
    \node[const, right=of b]    (b-comment)  {$\longleftarrow$ common bias} ;
    \foreach \i in {1,...,8,97,98,99}{
      \edge[-,color=\colorforb] {mu-\i} {b}
    }
    \edge[-,color=\colorforbb] {t-1} {b}
    \edge[-,color=\colorforbb] {t-2} {b}
    \edge[-,color=\colorforbb] {t-k} {b}
    % landmarks
    \node[latent, text=\colorforL, draw=\colorforL, right=of mu-99, xshift=0.5cm, yshift=-1.3cm] (L)  {$\bm{\mathcal{L}}$} ;
    %\edge[-,color=\colorforL] {mu-1} {L} % (no loop) good enough visually
    \foreach \i in {1,...,8,97,98,99}{
      \edge[-,color=\colorforL] {mu-\i} {L}
    }
    \node[const, right=of L]    (L-comment)  {$\longleftarrow$ landmarks} ;
  \end{pgfonlayer}

\end{tikzpicture}
  \caption{Variable dependencies for our model, \landSVM, where one SVM is learned per cluster but the local models interact through a common bias and $\mathcal{L}$, the set of landmarks.}
  \label{fig:gm-oursvm}
\end{figure}

\bibliographystyle{ieee}
\bibliography{paper}
\end{document}